\documentclass[11pt]{amsart}

\usepackage[a4paper,total={6in,10in}]{geometry}
\usepackage[inline]{enumitem}
\usepackage[utf8]{inputenc}
\usepackage[foot]{amsaddr}
\usepackage{amsmath,amsfonts,amsthm,amssymb}
\usepackage[ruled,vlined]{algorithm2e}
\usepackage{graphicx}
\usepackage{xcolor}
\usepackage{booktabs}
\usepackage[sortcites=true,backref=true]{biblatex}
\definecolor{red}{RGB}{200,16,46}

\usepackage[colorlinks=true,
linkcolor=red,
citecolor=red]{hyperref}

\graphicspath{{figures}}

\newtheorem*{theorem*}{Theorem}
\newtheorem{theorem}{Theorem}

\newtheorem{lemma}{Lemma}

\newcommand{\defeq}{\ensuremath{\mathrel{\mathop:}=}}

\newcommand{\mli}[1]{\mathit{#1}}
\DeclareMathOperator*{\minopt}{minimize}
\DeclareMathOperator*{\argmin}{argmin}

\newcommand\secref[1]{\S\ref{#1}}
\newcommand\figref[1]{Fig.~\ref{#1}}
\newcommand\thmref[1]{Thm.~\ref{#1}}
\newcommand\tabref[1]{Tab.~\ref{#1}}

\title{Automatic Classification of Deformable Shapes}
\author[Hossein Dabirian et al.]{H. Dabirian$^{1}$ \and R. Sultamuratov$^{2}$ \and J. Herring$^{2,4}$ \and C. El Tallawi$^{3}$ \and W. Zoghbi$^{3}$ \and A. Mang$^{2}$ \and R. Azencott$^{2}$}

\address[1]{Department of Electrical Engineering and Computer Science, University of Michigan, Ann Arbor, MI, USA}
\address[2]{Department of Mathematics, University of Houston, Houston, TX, USA (Email correspondences: \href{mailto:razencot@math.uh.edu}{razencot@math.uh.edu})}
\address[3]{Houston Methodist DeBakey Heart and Vascular Center, Houston Methodist Hospital, Houston, TX, USA}
\address[4]{Current Affiliation: Slingshot Aerospace, El Segundo, CA, USA}
\date{\today}

\begin{document}

\begin{abstract}
Let $\mathcal{D}$ be a dataset of smooth {3D}-surfaces, partitioned into  disjoint classes $\mli{CL}_j$, $j= 1, \ldots, k$. We show how \emph{optimized diffeomorphic registration} applied to large numbers of pairs $S,S' \in \mathcal{D}$ can provide descriptive feature vectors to implement automatic classification on $\mathcal{D}$, and generate classifiers invariant by rigid motions in $\mathbb{R}^3$. To enhance accuracy of automatic classification, we enrich the smallest classes $\mli{CL}_j$ by diffeomorphic interpolation of smooth surfaces between pairs $S,S' \in \mli{CL}_j$. We also implement small random perturbations of surfaces $S\in \mli{CL}_j$ by random flows of smooth diffeomorphisms $F_t:\mathbb{R}^3 \to \mathbb{R}^3$. Finally, we test our automatic classification methods on a cardiology data base of discretized mitral valve surfaces.
\end{abstract}

\maketitle

\section{Introduction}\label{intro}

In the past two decades, \emph{computational anatomy}~\cite{Grenander:1998a,Miller:2004a,Miller:2002a,Younes:2009a} has developed variational calculus algorithms and associated software tools for optimized diffeomorphic registration of deformable 3D shapes (volumes or surfaces). These techniques have been applied to many bio-medical datasets, mostly for quantified comparison of ``soft'' human organs across patients cohorts, or across time for specific patients~\cite{Hartman:2022a,Hsieh:2022a}. Diffeomorphic registration has been successfully applied to compare human brains across large sets of 3D-MRI brain images~\cite{Miller:2001a,Glaunes:2008a,Lee:2020a,Mussabayeva:2018a,Louis:2018a}. Various flavors of diffeomorphic registration approaches have also been used in cardiology to study image sequences of live beating hearts, for dynamic modeling of heart chambers, ventricle walls fibers, etc. (see, e.g., \cite{Bistoquet:2008a,Gorce:1996a,Mansi:2011a,Sundar:2009a,Delingette:2012a,Lombaert:2011a,Vadakkumpadan:2012a,Perperidis:2005a,Bai:2015a,Shen:2005a,Guigui:2022a}). Similarly, statistical shape models have been developed to not only capture shape variations of anatomical structures within a population but also aid image segmentation (see \cite{Ludke:2022a,Heimann:2009a,Ambellan:2019a,Davies:2002a,Davies:2002b}).

Co-authors of the present paper have previously applied diffeomorphic registration to a data base of live 3D-echocardiography of mitral valves ({\bf MV}), to reconstruct MV dynamics and compute strain distribution over mitral leaflets (see \cite{Azencott:2010a,Freeman:2014a,Jajoo:2011a,Zhang:2021a,ElTallawi:2021a,ElTallawi:2021b,ElTallawi:2019a}). The underlying database of smooth 3D-surfaces was cured and provided by Drs Carlos El Tallawi and William Zoghbi (Houston Methodist DeBakey Heart and Vascular Center). In the present paper, we use it as a benchmark to test the performances of the methodologies we develop here for automatic classification of smooth 3D-surfaces. More precisely, we will address the task of automatic discrimination between ``regurgitation'' MVs and ``normal'' MVs  by machine learning ({\bf ML}) via random forests ({\bf RF})~\cite{Cutler:2012a,Breiman:2001a}.

\subsection{Optimized Diffeomorphic Registration of 3D Surfaces}

Lucid monographs and reviews about (diffeomorphic) registration and shape matching are~\cite{Modersitzki:2004a,Sotiras:2013a,Modersitzki:2009a,Younes:2019a}. In this paper, the term ``smooth'' is shorthand for ``of class $C^{\infty}$.'' Let $\mathcal{S}$ be the set of all compact smooth 3D-surfaces properly embedded in $\mathbb{R}^3$, and having piecewise smooth boundaries. Given two surfaces $S,S'\in\mathcal{S}$, denote $\mathcal{F}(S,S')$ the set of time indexed flows $(F_t)_{t\geq 0}$ of smooth diffeomorphisms $F_t : \mathbb{R}^3 \to \mathbb{R}^3$ verifying $F_1(S) = S'$ and $F_0 = \operatorname{id}_{\mathbb{R}^3}$, where $\operatorname{id}_{\mathbb{R}^3} (x) = x$ denotes the identity map in $\mathbb{R}^3$. Let \[\int_0^1 \| \mathrm{d} F_t / \mathrm{d}t \|^2 \text{d}t\] be the \emph{kinetic energy} of the flow $(F_t)_{t\geq 0}$. Here, $\|w\|$ is a fixed Hilbert-norm for smooth vector fields $w$ on $\mathbb{R}^3$. Numerical minimization of the kinetic energy over all flows in $\mathcal{F}(S,S')$ has been implemented by several calculus of variations algorithms (see, e.g.,~\cite{Azencott:2010a,Beg:2005a,Freeman:2014a,Jajoo:2011a,Zhang:2021a,Polzin:2020a}; or automatic differentiation (see, e.g., \cite{Hsieh:2021a,Franccois:2021a,Bone:2018a})) and typically yields an (approximately) optimal flow $(F_t)$ having a nearly minimal kinetic energy denoted $\operatorname{KIN}(S,S')$. Then, the terminal $F_1$ is called an optimized diffeomorphic registration of $S$ onto $S'$. Aside from solving variational optimization problems, some recent work tries to estimate the diffeomorphic matching using ML techniques~\cite{Amor:2021a,Krebs:2019a,Sun:2022a,Yang:2017a,Bharati:2022a,Huang:2021a}. One key issue in this context is how these methods generalize to unseen data. Moreover, it has recently been shown that efficient hardware-accelerated implementations of variational methods~\cite{Brunn:2021a,Brunn:2020a,Himthani:2022a,Mang:2016b,Mang:2019a} yield runtimes that are almost en par with ML approaches.

\subsection{Automatic Classification of Smooth 3D-surfaces by Machine Learning}

In 3D medical imaging sequences, segmentation algorithms are often used to extract discretized smooth surfaces $S^1, \ldots, S^r \in \mathcal{S}$ modeling the outer surface of soft organs---such as human brains or hearts---as well as the boundaries of their internal chambers, cavities, ventricles, etc. To quantitatively compare these bio-medical soft 3D shapes across patients or across time, diffeomorphic registration has often been intensively and successfully used~\cite{Bauer:2022a,Bone:2020a,Krebs:2019a,Charon:2022a,Louis:2018a}.

The \textbf{main goal of this paper} is to explore how to efficiently combine diffeomorphic registration of smooth surfaces with supervised ML for automatic classification of smooth 3D-surfaces. Let $\mathcal{D}$ be a dataset of discretized smooth 3D-surfaces. When $\mathcal{D}$ is already partitioned into a finite number of disjoint classes, automatic class prediction by supervised ML is a natural goal. However, classical ML classifiers such as \emph{random forests} ({\bf RF})~\cite{Cutler:2012a,Breiman:2001a}, \emph{multi-layer perceptrons} ({\bf MLP}s)~\cite{Rosenblatt:1957a}, \emph{convolution neural networks} ({\bf CNN}s)~\cite{Gu:2018a}, or \emph{support vector machines} ({\bf SVM}s)~\cite{Cortes:1995a} require the characterization of every surface $S \in \mathcal{D}$ by a computable feature  vector $\operatorname{vec}(S) \in \mathbb{R}^N$, for some fixed $N\in\mathbb{N}$. We provide an innovative approach to construct feature vectors $\operatorname{vec}(S)$ boosting the accuracy of automatic classification for smooth surfaces. Note that for a surface $S$ discretized by a grid $[x_1 , \ldots, x_{n(S)}]$ of $n(S) \in \mathbb{N}$ points in $\mathbb{R}^3$, naive indexation of $S$ by the long vector $X(S) \in \mathbb{R}^{3n(S)}$ concatenating the $x_j$ has major flaws: Indeed, $X(S)$ does not necessarily have a fixed dimension, and randomly permuting the $x_j$ will still identify the same $S$ but yield a very different indexing vector $X'(S)$.

Generic descriptors for shape recognition include Fourier descriptors~\cite{Osowski:2002a}, geodesic moments~\cite{Luciano:2019a}, curvature information~\cite{Wu:1993a}, multiscale fractal dimensions~\cite{Torres:2004a,Plotze:2005a}, or local contour signatures~\cite{Junior:2018a}. Here, we generate several families of intrinsic feature vectors $\operatorname{vec}(S)$ invariant by rigid motions of $\mathbb{R}^3$. To this end, we first construct a family $\Delta$ of dissimilarities $\delta(S,S')$ computable for all pairs $(S,S')$ of smooth 3D-surfaces. All these dissimilarities are invariant by rigid motions of surfaces, and they are easily computed after numerical diffeomorphic registration $F_t : S \to S'$ with minimal kinetic energy  $\operatorname{KIN}(S,S')$. For this paper, the family $\Delta$ includes only the kinetic energy and all the quantiles $\mli{QUANT}_{\!\alpha}(S,S')$ of the isotropic strain values for the elastic deformation $F_t$. Geometric transformations of $S$ and $S'$ before diffeomorphic registration provide invariance of each dissimilarity $\delta \in \Delta$ under the group $G3$ generated by rotations, translations, and homotheties of $\mathbb{R}^3$. To generate natural feature vectors $\operatorname{vec}(S) \in \mathbb{R}^N$ invariant by the group $G3$, we first select a finite ``reference set'' $\mli{REF}$ of $r$ surfaces in $\mathcal{D}$ and a finite set $\mli{DIS} \subset \Delta$ of $s$ dissimilarities. We then set $N = r\times s$,  and define $\operatorname{vec}(S) \in \mathbb{R}^N$ as the vector with coordinates $\delta(S,\Sigma)$, where $\delta \in \mli{DIS}$ and $\Sigma\in \mli{REF}$.

For most ML classifiers, unbalanced  class sizes do degrade classification accuracy. To enrich any given class $\mli{CL}$ of smooth surfaces, we propose here a diffeomorphic interpolation algorithm: For any pair of surfaces $S,S' \in \mli{CL}$ optimized diffeomorphic registration of $S$ and $S'$ generates a continuous time indexed flow $(F_t)_{t\geq0}$ of diffeomorphisms with $F_0(S)=S$ and $F_1(S) =S'$. When the kinetic energy $\operatorname{KIN}(S,S')$ is small enough, the smooth surfaces $S_t = F_t(S)$ can be added to class $\mli{CL}$ as virtual new cases.

For ML classifiers, robustness is often enhanced when one enriches the training set by small random perturbations of existing training cases. To apply this approach to datasets of smooth surfaces, we generate continuous time flows $F_t$ of \emph{random smooth $R^3$ diffeomorphisms},  by time integration  of \emph{smooth} Gaussian random vector fields $V_t(x)$ indexed by time $t \geq 0$ and  $ x \in \mathbb{R}^3$. For small $t$, the smooth surfaces $F_t(S)$ are the small random deformations of $S$. Our numerical implementation involves a stochastic series expansion of smooth Gaussian vector fields  on $\mathbb{R}^3$, which required a sophisticated analysis of convergence speed.

\subsection{Contributions}

Our current work builds upon our prior work on designing effective algorithms for diffeomorphic shape matching and algorithmic reconstruction of unknown dynamics for MV leaflets~\cite{Azencott:2010a,Freeman:2014a,Jajoo:2011a,Zhang:2021a,ElTallawi:2021a,ElTallawi:2021b,ElTallawi:2019a}. Here, we are interested in developing a fully automatic computational framework for classification of deformable shapes, using tools derived from diffeomorphic registration of smooth 3D shapes~\cite{Bauer:2022a,Younes:2009a,Younes:2019a,Charon:2022a,Marslanda:2019a,Bauer:2019a}. Our main contributions are the following:
\begin{itemize}
\item We use diffeomorphic registration between pairs $S,S'$ of smooth surfaces to compute several generic families of dissimilarities $\operatorname{dis}(S,S')$ invariant by the group $G3$ generated by rigid motions and homotheties in $\mathbb{R}^3$.
\item For automatic classification of generic datasets of smooth 3D-surfaces, we construct large sets of $G3$-invariant feature vectors derived from the $G3$-invariant dissimilarities $\operatorname{dis}(S,S')$.
\item To re-balance class sizes in generic databases of smooth 3D-surfaces, we design two novel data enrichment approaches derived from diffeomorphic registrations. One approach is based on diffeomorphic shape interpolation, in order to extend the SMOTE enrichment technique based on linear interpolation of feature vectors, which is restricted to Euclidean distances. The second approach uses random diffeomorphic shape perturbations by automated simulations of smooth Gaussian random vector fields in $\mathbb{R}^3$.
\item For automatic classification in our benchmark dataset of MV surfaces, we  develop a localization scheme for dissimilarity computations, and tailor it to our discretized MV surfaces, in order to improve the discriminating powers of our feature vectors.
\item By implementing multiple RF classifiers based on our $G3$-invariant feature vectors, we perform a comparative importance analysis between nine groups of features to discover the most important of our dissimilarities, which turn out to be high quantiles of strain distributions.
\item We successfully apply all our preceding approaches to perform automatic RF-classi\-fi\-cation of 800 MV surfaces into two classes (``regurgitation'' vs ``normal'' cases), with high OOB accuracy.
\end{itemize}

\subsection{Outline}

In \secref{s:methods} we outline our methodology. We describe our benchmark dataset of smooth surfaces in \secref{s:data}. Our approach to diffeomorphic shape matching is described in \secref{s:diffreg}, where we formulate diffeomorphic registration as a classical variational problem, and the algorithmic approach we have implemented in our automatic solver for diffeomorphic surface registration. Going beyond the minimal kinetic energy, we outline the families of $G3$-invariant surface dissimilarities, which we derive from diffeomorphic registration. In \secref{s:autoclass}, we develop intrinsic families of $G3$-invariant feature vectors for automatic classification of 3D-surfaces. In \secref{s:enrichment} we outline two algorithms to enrich the classes of discretized 3D-surfaces—diffeomorphic shape interpolation and random diffeomorphic shape perturbations. In \secref{s:results} we discuss our experimental setup and present results for our benchmark dataset of MV surfaces, which includes the overall classification approach (see \secref{s:results-classification}), classification through RFs (see \secref{s:results-rfclassification}), as well as computing times (see \secref{s:computing-times}). We draw conclusions in \secref{s:conclusions}.

\section{Methods and Material}\label{s:methods}

\subsection{3D-Echocardiographies of Human Mitral Valves}\label{s:data}

In a long term collaboration between the research groups of William Zoghbi, MD, Methodist Hospital (DeBakey Heart and Vascular Center) and Robert Azencott, University of Houston (Mathematics), our team has studied echocardiographic images acquired in vivo from 150 cardiology patients with potential MV complications. For these  patients, transesophageal echocardiography provides dynamic 3D views of their MVs, at rates of roughly 25 frames per heart cycle. Fast computerized segmentation of these 3D-images, a TOMTEC-Philips software extracts a 3D-snapshot of the two MV leaflets surfaces per 3D-frame. For frame time $t$, the extracted 3D-snapshot displays the anterior ({\bf AL}) and posterior leaflets ({\bf PL}), denoted by $\mli{AL}(t)$ and $\mli{PL}(t)$, of the MV as smooth surfaces discretized by a dense grid of 800 points per leaflet. This data base of 3\,500 discretized 3D-snapshots of MV leaflets was prepared and annotated by Dr. El-Tallawi, Methodist Hospital (DeBakey Heart and Vascular Center).

During each heart cycle, the leaflets $\mli{AL}(t)$ and $\mli{PL}(t)$ close the MV at midsystole ($t = t_{\text{midsys}}$) and open the MV at endsystole ($t = t_{\text{endsys}}$). These two leaflets define a deformable connected 3D-surface $\mli{MV}(t)$ bounded by a flexible ring (the ``annulus''). At $t = t_{\text{midsys}}$, the leaflets $\mli{AL}(t)$ and $\mli{PL}(t)$ are in full contact along the \emph{coaptation line} in order to tightly close the MV. The leaflets open the MV progressively until endsystole, then remains fully open during diastole, and start closing again at beginning of systole. When the MV is open, the coaptation line is split into two curved boundary segments $\partial \mli{AL}(t)$, $\partial \mli{PL}(t)$, sharing the same endpoints. We show representative patient data in \figref{f:mvdata}.

\begin{figure}
\includegraphics[width=0.9\textwidth]{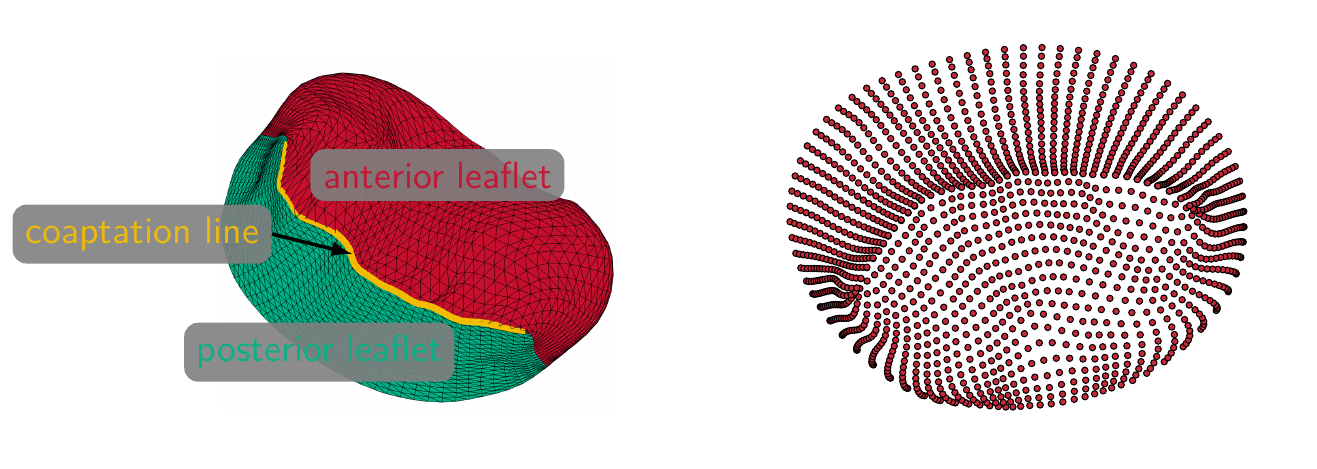}
\caption{Left: Anatomical regions for MV dataset. We highlight the posterior leaflet (teal color), the anterior leaflet (red color), and the coaptation line (gold). On the right we show the discretized shape $S$ represented by a grid $[x_1 , \ldots, x_{n(S)}]$ of $n(S) \in \mathbb{N}$ points in $\mathbb{R}^3$.\label{f:mvdata}}
\end{figure}

In previous studies~\cite{Azencott:2010a,Zhang:2021a,ElTallawi:2021a,ElTallawi:2021b,ElTallawi:2019a}, automatic \emph{diffeomorphic registration} of the MV leaflets between $t = t_{\text{midsys}}$ and $t = t_{\text{endsys}}$ was developed and systematically implemented for this data base of 3D-image sequences, in order to compute the intensities and spatial distribution of the tissue strain induced by MV deformation at each heartbeat. This first research project aimed to provide cardiologists with patient specific displays of MV leaflet strain intensities, as a potential aid to evaluate/compare MV clinical cases~\cite{ElTallawi:2021a,ElTallawi:2021b,ElTallawi:2019a}. Our diffeomorphic registration solver was installed and tested at Methodist Hospital (DeBakey Heart and Vascular Center). Our current version of the solver performs these diffeomorphic registrations in less than 2\,min per patient on a standard laptop (Matlab implementation), for pairs of surfaces discretized by 1\,600 points each. In the present paper, we develop diffeomorphic deformation techniques for automatic classification of soft smooth shapes. We have used our database of 3D MV snapshots as a benchmark to implement and test our approach to automatic classification of smooth surfaces, by applying it to automatic discrimination between two classes of patients, namely \emph{regurgitation} cases versus \emph{normal} cases.

\subsection{Diffeomorphic Registration of 3D-Surfaces}\label{s:diffreg}

Below, we describe our approach for diffeomorphic shape matching.

\subsubsection{Diffeomorphic Deformations in $\mathbb{R}^3$}

Recall the basic mathematical formalization of computational anatomy~\cite{Beg:2005a,Younes:2019a,Grenander:1998a,Miller:2004a}. Fix a scale parameter $s>0$ and let $K_s(x,y)$ be the positive definite radial kernel
\begin{equation}\label{e:kernelKs}
K_s(x,y)=\exp(-\|x-y\|^2 / s^2) \;\;\text{for all} \;\; x,y \in \mathbb{R}^3.
\end{equation}

We call any smooth map $x \mapsto w_x$ from $\mathbb{R}^3$ to $\mathbb{R}^3$ such that $w_x$ and all its derivatives tend to 0 as $\|x\| \to \infty$ a \emph{smooth vector field} $w$ on $\mathbb{R}^3$. For any such $w$, define the norm $\| w \|$ by
\begin{equation}\label{e:normKs}
\|w\|^2 = \int_{\mathbb{R}^3} \int_{\mathbb{R}^3} K_s(x,y) \langle w_x,w_y \rangle_{\mathbb{R}^3}\, \mathrm{d}x \,\mathrm{d}y.
\end{equation}

\noindent Endowed with this norm, the vector space $V$ of smooth vector fields becomes a Hilbert space. Call \emph{velocity flow} any set $v = (v_t)$ of time-indexed smooth vector fields $v_t \in V$ such that $t \mapsto v_t$ is a Lipschitz continuous map from $[0,1]$ into $V$. Denote $\mathcal{V}$ the Hilbert space of all velocity flows $v = (v_t)$ having finite \emph{kinetic energy} $\operatorname{kin}(v)$  defined by
\[
\operatorname{kin}(v) = \int_0^1 \| v_t \|^2 \,\mathrm{d}t.
\]

Call \emph{smooth deformation} of $\mathbb{R}^3$ any time indexed flow $(F_t)$, $t \in [0, 1]$, of smooth diffeomorphisms $F_t$ from $\mathbb{R}^3$ to $\mathbb{R}^3$, such that $F_0$ is the identity map $\text{id}_{\mathbb{R}^3} : \mathbb{R}^3 \to \mathbb{R}^3$. As shown in~\cite{Dupuis:1998a,Trouve:1998a}, for any velocity flow $v= (v_t)$ in $\mathcal{V}$ there is a unique smooth deformation $(F_t)$ of $\mathbb{R}^3$ solving the ordinary differential equation ({\bf ODE})
\[
\frac{\mathrm{d}F_t}{\mathrm{d}t} = v_t(F_t) \;\; \text{for almost all} \;\; t \in [0,1].
\]

\noindent Then, $\text{kin}(v)$ will also be called the kinetic energy of the smooth deformation $(F_t)$.

\subsubsection{Diffeomorphic Registration of Surfaces}

To compare two surfaces $S$ and $\Sigma$ in $\mathcal{S}$, one seeks a smooth deformation $(F_t)$ having \emph{minimal kinetic energy} among all deformations verifying $F_1(S) = \Sigma$. This requires finding a velocity flow $v = (v_t)$ in $\mathcal{V}$ and an associated smooth deformation flow $(F_t)$ solving the variational problem
\begin{subequations}\label{e:constopt}
\begin{equation}
\minopt_{v \in \mathcal{V},\; (F_t) \in \mathcal{F}} \;\;\text{kin}(v) \\
\end{equation}

\noindent under the nonlinear constraints
\begin{align}
\frac{\mathrm{d}F_t}{\mathrm{d}t} & = v_t(F_t) \;\; \text{for almost all} \;\; t \in [0,1],\\
F_0 & = \text{id}_{\mathbb{R}^3},\\
F_1(S) & = \Sigma.
\end{align}
\end{subequations}

To numerically solve the variational problem in \eqref{e:constopt} after space and time discretization of $S$, $\Sigma$, and time interval $[0,1]$, one has to relax the rigid matching constraint $F_1(S) = \Sigma$ as in~\cite{Beg:2005a,Trouve:1998a,Glaunes:2004a,Azencott:2010a,Zhang:2021a}. This can be achieved by relaxing the constraint $F_1(S) = \Sigma$ to $F_1(S) \approx \Sigma$, introducing a shape matching dissimilarity between $F_1(S)$ and $\Sigma$, as described in the following section.

\subsubsection{Kernel-Based Dissimilarity between Smooth Surfaces}

The set $\mathcal{S}$ of compact smooth 3D-surfaces with boundaries can be endowed with many natural shape matching dissimilarities. For a fast numerical solution of the variational problem \eqref{e:constopt},  efficient differentiable shape dissimilarities are provided, as we now outline, via the self-reproducing Hilbert space ({\bf RKHS}) associated to the radial Gaussian kernel $Q(x,y) = K_{\tau}(x,y)$ defined in \eqref{e:kernelKs} with any fixed scale parameter $\tau >0$. The space of bounded Radon measures $\mu$ on $\mathbb{R}^3$ is a Hilbert space $H$ for the norm $\|\mu\|$ defined by
\begin{equation}\label{e:RM}
\|\mu\|^2 = \int_{\mathbb{R}^3} \int_{\mathbb{R}^3} Q(x,y) \,\mathrm{d}\mu(x)\,\mathrm{d}\mu(y).
\end{equation}

The Lebesgue measure of $\mathbb{R}^3$ induces on each surface $S \in \mathcal{S}$ a Riemannian surface element $\mathrm{d}\mu_S (z)$, which determines a bounded Radon measure $\mu_S \in H$ with support equal to $S$. We rescale $\mu_S$ by imposing $\mu_S(S) = 1$. Define the Hilbertian \emph{shape matching dissimilarity} $\operatorname{HILB}(S,\Sigma)$ between any two surfaces $S, \Sigma \in \mathcal{S}$ by
\[
\operatorname{HILB}(S,\Sigma) = \| \mu_S - \mu_{\Sigma} \|^2.
\]

\noindent When $S, \Sigma$ are discretized by two finite grids of points $x_n \in S$, $1\leq n \leq N$ and $y_m \in \Sigma$, $1\leq m \leq M$, one approximates $\mu_S$ and $\mu_{\Sigma}$ by sums of Dirac masses
\[
\nu_S = \frac{1}{N} \sum_{n=1}^N \delta_{x_n}
\qquad \text{and} \qquad
\nu_{\Sigma} = \frac{1}{M} \sum_{m=1}^M \delta_{y_m}.
\]

\noindent This approximates $\text{HILB}(S,\Sigma)$ by $\|\nu_S - \nu_{\Sigma}\|^2$ which is a simple linear combination of all terms $Q(x_n,x_n')$, $Q(y_m,y_m')$, $Q(x_n,y_m)$.

\subsubsection{Relaxed Cost Functions for Numerical Diffeomorphic Registration}

To soften the rigid matching constraint $F_1(S) = \Sigma$ in \eqref{e:constopt}, fix a positive weight $\lambda$. Then, seek a velocity flow $v = (v_t)$, $v_t \in \mathcal{V}$, and an associated diffeomorphic flow $(F_t)$, $F_t\in\mathcal{V}$, which solves the relaxed variational problem
\begin{subequations}\label{e:varopt}
\begin{align}
\minopt_{v \in \mathcal{V},\;(F_t) \in \mathcal{F}} & \;\;\text{kin}(v)  + \lambda\, \text{HILB}(F_1(S),\Sigma) \\
\begin{aligned}
\text{subject to}\\\\
\end{aligned}
& \;\;
\begin{aligned}
\frac{\mathrm{d}F_t}{\mathrm{d}t} & = v_t(F_t) \;\; \text{for almost all} \;\; t \in [0,1],\\
F_0 & = \text{id}_{\mathbb{R}^3}.
\end{aligned}
\end{align}
\end{subequations}

For fixed $\lambda$, after space-time discretization of $S$, $\Sigma$, and $[0,1]$, the search for a vector field flow $v = (v_t)$ minimizing the cost functional is then implementable numerically by various gradient descent techniques~\cite{Glaunes:2004a,Beg:2005a,Azencott:2010a,Freeman:2014a,Jajoo:2011a,Zhang:2021a,Mang:2015a,Mang:2017a,Mang:2017b}).  Our numerical implementation of diffeomorphic registration for 3D-surfaces is outlined further in \secref{s:solver}, and provides a good approximation of the minimal kinetic energy $\operatorname{KIN}(S,\Sigma)$. The theoretically valid symmetry relation $\operatorname{KIN}(S,\Sigma) = \operatorname{KIN}(\Sigma, S)$ is only approximately true for numerical estimates, so that the average $(\operatorname{KIN}(S,\Sigma) + \operatorname{KIN}(\Sigma, S))/2$ improves numeric accuracy (but doubles computing times).

\subsubsection{Strain Analysis}\label{s:strain}

After computing a nearly optimal diffeomorphic registration $F=(F_t)$ matching two surfaces $S, \Sigma \in \mathcal{S}$, the terminal $\mathbb{R}^3$-diffeomorphism $f = F_1$ is a smooth invertible map from $S$ onto a smooth surface $\hat{S} = f(S)$, with very small Hausdorff distance $\operatorname{HAUS}(\hat{S}, \Sigma)$, where
\begin{equation}\label{e:hausdorff}
\textstyle\operatorname{HAUS}(S_1,S_2)= \max_{x\in S_1} \min_{y \in S_2} \|x -y\|.
\end{equation}

The minimal kinetic energy $\operatorname{KIN}(S,\Sigma)$ involves averages of squared velocities over the whole of $S$, and hence only provides a global dissimilarity between $S$ and $\Sigma$. Strain analysis of $f$, which we now outline as in \cite{Zhang:2021a}, will instead generate the spatial distribution of local distortions between $S$ and $\Sigma$: Fix any point $x \in S$ and let $y = f(x) \in f(S)$. Denote $T_x$, $T_y$ the tangent spaces to $S$, $\hat{S}$ at $x$ and $y$, respectively, endowed with local surface metrics. Since $f : S \to f(S)$ is a smooth bijection, the differential $f'(x)$ determines an invertible $2 \times 2$ linear map $J_x : T_x \to T_y$. For any tangent vector $u \in T_x$ with $\|u\| = 1$, the \emph{directional strain} at $x$ induced by deformation $f$ in direction $u$ is the length dilation (or contraction) factor $\mli{dirSTR}(x,u) = |J_x u|$. Let $J_x^*: T_y \to T_x$ be the transpose of $J_x$, and denote $m_x \leq M_x$ the eigenvalues of the positive definite $2 \times 2$ matrix $J_x^* J_x$. The minimal and maximal directional strains around $x$ are equal to $\sqrt{m_x}$ and $\sqrt{M_x}$, respectively. In general, one has $m_x < M_x$ and the directional strain at $x$ depends on $u$. As in \cite{Zhang:2021a}, to avoid this anisotropy, we focus on the \emph{isotropic strain}
\[
\mli{isoSTR}_x = \sqrt{m_x M_x} = \sqrt{| \det(J_x) |}.
\]

\noindent Indeed, $\mli{isoSTR}_x$ has a simple geometric interpretation. For  fixed $x \in S$, and any open patch $U_x \subset S$ around $x$, define the ratio of surface areas
\[
\operatorname{rat}(U_x)= \operatorname{area}(f(U_x)) / \operatorname{area}(U_x).
\]

\noindent Then, $\operatorname{rat}(U_x)$ tends to $\mli{isoSTR}_x^2$ when the diameter of $U_x$ tends to 0. This provides the following fast numerical approximation of isotropic strain: After discretization of $S$ by a finite grid $\mli{grid}_S$ with small mesh size, and triangulation of $\mli{grid}_S$, one simply sets $U_x$ to be the union of all triangles with vertex $x$. Since $\mli{isoSTR}_x$ is a dimensionless average length dilation (or contraction) factor around $x$, we convert it into an \emph{isotropic strain intensity}
\[
\mli{isi}_x = |\mli{isoSTR}_x - 1|.
\]

\noindent Clearly, $\mli{isi}_x$ quantifies the intensity of local lengths deformation around $x$ by the diffeomorphic registration $f = F_1$ from $f: S \to f(S) \approx \Sigma$.

As above, denote $\mu_S$ the probability distribution induced on $S$ by the $\mathbb{R}^3$-Lebesgue measure. When $x \in S$ is selected at random---with probability distribution $\mu_S$---the distribution of the random values $\mli{isi}_x$ is a probability $\operatorname{isi}(S,\Sigma)$ on $\mathbb{R}^+$. For each percentile $0 < \alpha < 1$, the quantile $q_{\alpha}$ of $\operatorname{isi}(S,\Sigma)$ can be viewed as a dissimilarity $\mli{QUANT}_{\alpha}(S, \Sigma)$ between $S$ and $\Sigma$. These dissimilarities are well approximated via quantiles of the finite samples $\mli{ist}_x$, $x \in \mli{grid}_S$, when $\mli{grid}_S$ has small mesh size. One can symmetrize $\mli{QUANT}_{\alpha}$ by averaging its values for $(S, \Sigma)$ and $(\Sigma, S)$.

\subsubsection{Diffeomorphic Registration Software}\label{s:solver}

Prior versions of the diffeomorphic registration algorithms described in this section have been developed and tested in~\cite{Azencott:2010a,Zhang:2021a}. These versions were implemented in MATLAB and applied to study 3D-echocardiography data of MV patients~\cite{ElTallawi:2021a,ElTallawi:2021b,ElTallawi:2019a,Zekry:2016a,Zekry:2018a,Zekry:2012a}. As in our earlier papers, we implement a \emph{discretize-then-optimize} approach for solving the optimization problem~\eqref{e:varopt}. The target shape $\Sigma$ and the template shape $S$ are represented by two grids of points $y_1 = \{y_1^i\}_{i=1}^{n}$ and $y_0 = \{y_0^i\}_{i=1}^{m}$ in $\mathbb{R}^3$, respectively. We model the velocity vector fields $v_t$ as vectors belonging to a RKHS defined by a Gaussian kernel. Consequently, we write
\[
v_t(z) = \sum_{i=1}^{m} K_{\sigma}(x_i(t),z) a_i(t)\quad \text{for all}\;\; z \in \mathbb{R}^3.
\]

\noindent Here, $K_{\sigma} : \mathbb{R}^3 \times \mathbb{R}^3 \to \mathbb{R}$ corresponds to the kernel in \eqref{e:kernelKs}. The coefficients $\{a_i\}_{i=1}^m$,  $a_i \in \mathbb{R}^3$, are then the new controls of the discrete optimization problem. We discretize the time interval $[0,1]$ by a nodal grid, resulting in $q$ equispaced intervals. We collect the associated coefficients $a_i^j \defeq a_i(t_j)$, $j = 0,\ldots,q$, in a vector (lexicographical ordering) $a$ of size $c = 3m(q+1)$. The ODE (i.e., the flow equation) in \eqref{e:varopt} is discretized using a first order explicit Euler method. As for the control $a$, we collect the associated states (deformed shape) in a concatenated vector $x$ of size $s = 3m(q+1)$. With this, we can represent the forward Euler step as a linear system
\[
\begin{bmatrix} G^x \;\; G^a \end{bmatrix}
\begin{bmatrix} x \\ a \end{bmatrix}
= g
\]

\noindent with state-control vector $(x,a) \in \mathbb{R}^{6m(q+1)}$. The moderate value $q = 4$ achieved a good compromise between numerical accuracy and computing time for all numerical diffeomorphic surface registrations considered here. Moreover, in our benchmark application, $m = n = 800$.

With slight abuse of notation we arrive at the discretized optimization problem
\begin{subequations}
\label{e:discvaropt}
\begin{align}
\minopt_{a\in\mathbb{R}^c,\; x \in \mathbb{R}^s}
 & \;\;\text{kin}(a)  + \lambda\, \text{HILB}(x^q,y_1) \\
\text{subject to} &\;
\begin{bmatrix} G^x \;\; G^a \end{bmatrix}
\begin{bmatrix} x \\ a \end{bmatrix}
= g.
\end{align}
\end{subequations}

\noindent We refer to \cite{Zhang:2021a} for the precise form of the operators that appear in \eqref{e:discvaropt}. To solve \eqref{e:discvaropt} we apply an operator splitting strategy typically referred to as alternating direction method of multipliers~\cite{Boyd:2011a,ODonoghue:2013a,Parikh:2013a} or Douglas--Rachford splitting~\cite{Douglas:1956a}. In its modern form, this algorithm was introduced in~\cite{Glowinski:1975a,Gabay:1976a}. The consensus form of operator splitting for the discrete control problem in~\eqref{e:discvaropt} at iteration $k$ is given by
\begin{subequations}
\begin{align}
\begin{bmatrix}
x_{k+1} \\ a_{k+1}
\end{bmatrix}
&=
\argmin_{x,\, a}
\left(
\text{ind}_{\mathcal{C}}(x,a) + \text{kin}(a)
+ \frac{\chi}{2}
\left\|\,
\begin{bmatrix} x \\ a \end{bmatrix}
-
\begin{bmatrix} \tilde{x}_k \\ \tilde{a}_k \end{bmatrix}
-
\begin{bmatrix} u_k \\ w_k \end{bmatrix}
\,\right\|_2^2
\right),
\label{e:stepkin}
\\
\begin{bmatrix}
\tilde{x}_{k+1} \\ \tilde{a}_{k+1}
\end{bmatrix}
&=
\argmin_{\tilde{x},\, \tilde{a}}
\left(
\lambda \text{HILB}(\tilde{x}^q,y_1)
+ \frac{\chi}{2}
\left\|\,
  \begin{bmatrix} \tilde{x} \\ \tilde{a} \end{bmatrix}
- \begin{bmatrix} x_{k+1} \\ a_{k+1} \end{bmatrix}
+ \begin{bmatrix} u_k \\ w_k \end{bmatrix}
\,\right\|_2^2
\right),
\label{e:stepdist}
\\
   \begin{bmatrix} u_{k+1} \\ w_{k+1} \end{bmatrix}
&= \begin{bmatrix} u_k \\ w_k \end{bmatrix}
+  \begin{bmatrix} \tilde{x}_{k+1} \\ \tilde{a}_{k+1} \end{bmatrix}
-  \begin{bmatrix} x_{k+1} \\ a_{k+1} \end{bmatrix}.
\label{e:consens}
\end{align}
\end{subequations}

\noindent Here, $\chi > 0$ is an algorithm parameter and $\text{ind}_{\mathcal{C}}$ represents an indicator function for the set $\mathcal{C} \subseteq \mathbb{R}^{6m(q+1)}$ of state-control pairs $(x,a)$ that satisfy the discretized dynamical system in \eqref{e:discvaropt}, i.e.,
\[
\mathcal{C} \defeq \left\{\begin{bmatrix} x \\ a \end{bmatrix} \mid
\begin{bmatrix} G^x \;\; G^a \end{bmatrix}
\begin{bmatrix} x \\ a \end{bmatrix}
= g
\right\}.
\]

We solve the first order optimality conditions of subproblem \eqref{e:stepkin} for the state-control vector $(x,a)$ using a matrix-free, preconditioned conjugate gradient method. Given the solution of \eqref{e:stepkin}, we solve \eqref{e:stepdist} for the state-control vector $(\tilde{x},\tilde{a})$ using a matrix-free Newton--Krylov method~\cite{Nocedal:2006a}. The last step in \eqref{e:consens} represents an update of the dual variables $(u,w)$ associated with the consensus constraint $(x,a) = (\tilde{x},\tilde{a})$. We note that~\cite{Zhang:2021a} used direct methods to solve the individual subproblems instead of iterative methods. Using iterative methods allowed us to reduce the runtime by roughly a factor of two. A more detailed discussion of this solver is beyond the scope of the present paper and will be provided elsewhere. For the sake of the present work we will consider it a black-box method that provides us with efficient and fast diffeomorphic registration of 3D-surfaces. We terminate this algorithm when the discretized, censored Hausdorff distance between the deformed shape $x^q$ and the reference shape $y_1$ is of the order of the surface discretization mesh size  (see \cite{Azencott:2010a,Zhang:2021a} for details).

\subsection{Shape Dissimilarities and Automatic Shape Classification}\label{s:autoclass}

\subsubsection{Automatic Shape Classification: The Need for Intrinsic Feature Vectors}

Consider any dataset $\mathcal{D} \subset \mathcal{S}$ of smooth 3D-surfaces, partitioned into a finite set of classes $\mli{CL}_j$. We implement automatic classification in $\mathcal{D}$ by supervised ML. All well known classifiers such as MLPs, RFs or SVMs require describing each surface $S \in \mathcal{D}$ by ``natural''  feature vectors $\operatorname{vec}(S)$ belonging to an Euclidean space of \emph{fixed} dimension (or to some fixed self-reproducing Hilbert space). However, in practice each surface $S \in \mathcal{D}$ is available only through discretization by a finite grid $\mli{grid}_S$ of points $x_n \in S$, and the cardinality $\mli{card}_S$ of $\mli{grid}_S$ often varies with $S$. Naive direct indexation of $S$ by the vector $X(S) = [x_1, x_2, \ldots, x_n, \ldots]$, is mathematically not sound, since the dimension $3 \mli{card}_S$ of $X(S)$ may vary with $S$, and any permutation of the $x_n$ would radically modify $X(S)$, while still defining the same discretized surface $S$. Intrinsic feature vectors $\operatorname{vec}(S)$ describing smooth shapes $S$ discretized by finite grids $\mli{grid}_S$ should at least remain stable under permutations of these grids, and exhibit some natural consistency when the mesh size of $\mli{grid}_S$ tends to 0. Another natural requirement in many biomedical applications is invariance of $\operatorname{vec}(S)$ when $S$ is replaced by $\kappa.S$, where $\kappa$ is any rigid motion in $\mathbb{R}^3$.

\subsubsection{Families of Dissimilarities Invariant to Rigid Motions}\label{s:families}

To generate intrinsic feature vectors, we start by introducing multiple types of \emph{dissimilarities} $\operatorname{dis}(S,\Sigma) = \operatorname{dis}(\Sigma, S) \geq 0$ between pairs of surfaces. These dissimilarities are not necessarily distances; indeed, some of them are natural squared distances for instance but in all our uses of dissimilarities one does not need the triangular inequalities verified by distances. For the purpose of automatic classification of smooth surfaces, our generic  dissimilarities are typically of the form
\begin{equation}\label{e:dissimilarities}
    \operatorname{dis}(S,\Sigma) = u\left[d(S,\Sigma)\right],
\end{equation}

\noindent where $u: \mathbb{R}^+\to \mathbb{R}^+$ is any continuous increasing function  verifying $u(0)=0$ and $d(S,\Sigma)$ is any bona fide distance between smooth surfaces.

We have already defined four explicit types of dissimilarities $\operatorname{dis}(S,\Sigma)$ between pairs of surfaces $S,\Sigma \in \mathcal{S}$, namely:
\begin{enumerate}
\item\label{i:HAUS} the Hausdorff distance $\operatorname{HAUS}(S,\Sigma)$,
\item\label{i:HILB} the squared Hilbert distance $\operatorname{HILB}(S,\Sigma)$,
\item\label{i:KINE} the kinetic energy $\operatorname{KIN}(S,\Sigma)$, and
\item\label{i:QUAN} the strain quantiles $\operatorname{QUANT}_{\alpha}(S,\Sigma)$.
\end{enumerate}

\noindent Let $\text{SE}(3)$ be the group of rigid motions in $\mathbb{R}^3$, generated by translations and rotations. Each one of these basic four dissimilarities $\text{dis}(S,\Sigma)$ is invariant by rigid motion, i.e., verifies
\begin{equation}\label{e:invar}
\operatorname{dis}(S,\Sigma) = \operatorname{dis}(\rho.S, \rho.\Sigma) \;\; \text{for all} \;\; \rho \in \text{SE}(3).
\end{equation}

\noindent This is an interesting property in the context of automatic classification for soft organ shapes observed across multiple patient groups, since one expects the true class of a soft shape $S$ to be invariant by all rigid motions of $S$. Within our benchmark dataset of discretized MV surfaces, we also wanted to mitigate the impact of patient height and weight by applying to each surface $S$ an adequate homothetic transformation. This led us to construct strongly invariant dissimilarities as follows.

\subsubsection{Strongly Invariant Dissimilarities}\label{s:stronglyinvariant}
For any surface $S \in \mathcal{S}$, denote $c_S$ its center of mass and $\operatorname{ten}_S$ its $3 \times 3$ tensor of inertia, which are given by
\[
c_S = \int_{x \in S}  x\; \mathrm{d}\mu_S(x)
\qquad \text{and}\qquad
\mli{ten}_S
= \int_{x \in S} \int_{y \in S} (x -c_S)(y-c_S)^* \,\mathrm{d}\mu_S(x)\,  \mathrm{d}\mu_S(y),
\]

\noindent where $x$, $c_S$ are column vectors and $*$ denotes transposition. When $S$ is discretized by a finite grid $\mli{grid}_S$, the intrinsic measure $\mu_S$ induced on $S$ by $\mathbb{R}^3$-Lebesgue measure is simply the average of Dirac masses carried by the points of $\mli{grid}_S$. The matrix $\operatorname{ten}_S$ has positive eigenvalues $\lambda_1 \leq \lambda_2 \leq \lambda_3$ and unit length eigenvectors $\eta_1, \eta_2, \eta_3$.

Fix an orthonormal basis $e_1,e_2,e_3$ in $\mathbb{R}^3$. For each $S \in \mathcal{D}$ define the $\mathbb{R}^3$-rotation $\mli{rot}_S$ and the rigid motion $\rho_S$ by
\begin{equation}\label{e:rhoS}
\mli{rot}_S \eta_j = e_j \;\;\text{for}\;\; j = 1, 2, 3,
\qquad \text{and} \qquad
\rho_S y = \mli{rot}_S (y-c_S) \;\; \text{for}\;\; y \in \mathbb{R}^3.
\end{equation}

To each surface $S$ in our dataset $\mathcal{D}$, we will first associate the surface $S'= \rho_S.S$. Note that $c_{S'} = 0$, and that the inertia tensor $\mli{ten}_{S'}$ has $\{e_1, e_2, e_3\}$ as ordered eigenvectors. Then, we transform $S'$ into $S''= h_S. S'$, where $h_S$ is the homothety centered at 0 and such that $\operatorname{area}(S'') = 1$. The linear transformation $g_S = h_S \circ \rho_S$ belongs to the group $G3$ generated by the rotations, translations, and homotheties from $\mathbb{R}^3$ to $\mathbb{R}^3$, and we will systematically replace $S$ by $S''= g_S.S$. Let  $\kappa$ be either a translation, or a rotation, or an homothety from $\mathbb{R}^3$ to $\mathbb{R}^3$. As is directly verified in each one of these three cases, the surface  $(\kappa.S)''= g_{\kappa.S}.(\kappa.S)$ is actually identical to $S''$. Hence the property $(\kappa.S)'' = S''$ will also hold for all linear transformations $\kappa \in G3$.

Any dissimilarity $\operatorname{dis}(S,\Sigma)$ defined for all pairs $S,\Sigma$ in $\mathcal{S}$ and verifying the rigid motion invariance \eqref{e:invar}, will naturally define a standardized dissimilarity $\operatorname{dis}^{+}$ by
\begin{equation}\label{e:disrho}
\operatorname{dis}^{+}(S,\Sigma) = \operatorname{dis}(g_S.S, g_{\Sigma}.\Sigma) = \operatorname{dis}(S'',\Sigma'').
\end{equation}

\noindent Thanks to \eqref{e:invar}, the dissimilarity $\operatorname{dis}^{+}(S,\Sigma)$ does not depend on the choice of orthonormal basis $e_1,e_2,e_3$. Moreover, for any pair of transformations $\kappa_1, \kappa_2 \in G3$, we have---as seen above---$(\kappa_1.S)'' = S''$ and $(\kappa_2.\Sigma)'' = \Sigma''$, so that \eqref{e:disrho} implies
\[
\operatorname{dis}^{+}(S,\Sigma)
= \operatorname{dis}^{+}(\kappa_1.S,\kappa_2.\Sigma) \;\;
\text{for all} \;\; \kappa_1,\kappa_2  \in G3.
\]

To each one of the dissimilarities $\operatorname{dis}$ listed above as items \ref{i:HAUS} through \ref{i:QUAN}, we apply the preceding construction, to generate a corresponding  strongly invariant dissimilarity $\operatorname{dis}^{+}$. In what follows, we will adopt the following \emph{simplified notations}:
\begin{enumerate}
\item\label{i:haus}  $\operatorname{haus}(S,\Sigma) = \operatorname{HAUS}^{+}(S,\Sigma)$,
\item\label{i:hilb} $\operatorname{hilb}(S,\Sigma) = \operatorname{HILB}^{+}(S,\Sigma)$,
\item\label{i:kine} $\operatorname{kin}(S,\Sigma) = \operatorname{KIN}^{+}(S,\Sigma)$, and
\item\label{i:quan} $\operatorname{quant}_{\alpha} (S,\Sigma) = \operatorname{QUANT}_{\alpha}(S,\Sigma)$.
\end{enumerate}

\subsubsection{Construction of Intrinsic Feature Vectors for Shape Classification}\label{s:intrinsic}

Let $\mathcal{D}$ be a benchmark set of smooth surfaces partitioned into several disjoint classes $\mli{CL}_j$ of surfaces. To define intrinsic feature vectors describing these surfaces, we select (and fix) in $\mathcal{D}$ a finite \emph{reference set} $\mli{REF}= \{\Sigma_1 \ldots \Sigma_N\}$ of surfaces. In our applications below, $\mli{REF}$ will simply be either the set of all training cases or a large subset of the set of ``normal'' cases.

Select any one of the strongly invariant dissimilarities $\operatorname{dis}(S,\Sigma)$  listed as items \ref{i:HAUS} through \ref{i:QUAN} above. Then, characterize any surface $S \in \mathcal{D}$ by the $N$-dimensional \emph{feature vector} $\operatorname{vec}(S)$ having coordinates
\[
\operatorname{vec}(S)_k = \operatorname{dis}(S, \Sigma_k).
\]

These feature vectors verify the $G3$ invariance $\operatorname{vec}(S) = \operatorname{vec}(\kappa.S)$ for all linear transformations $\kappa \in G3$. We extend this approach by selecting $p$ distinct dissimilarities $\operatorname{dis}_1, \ldots, \operatorname{dis}_p$  among all the strongly invariant dissimilarities listed above as items \ref{i:HAUS} through \ref{i:QUAN}. Each associated $\operatorname{dis}_i$ defines---as above---feature vectors $\operatorname{vec}^i(S) \in \mathbb{R}^N$, which can naturally be concatenated into the $pN$ dimensional feature vector
\[
\operatorname{VEC}(S) = \left[ \operatorname{vec}^1(S),  \ldots, \operatorname{vec}^p(S)\right],
\]

\noindent which also verifies $\operatorname{VEC}(\kappa.S) = \operatorname{VEC}(S)$ for all $\kappa \in G3$. Note that the feature vectors just constructed essentially verify the desirable properties outlined above for intrinsic feature vectors.

To test and validate our intrinsic feature vectors approach for indexing  smooth  surfaces, we have applied it to the automatic classification of MVs by RF classifiers (see \secref{s:results} below).

\subsection{Data Enrichment}\label{s:enrichment}

Since our benchmark dataset included unbalanced classes, we first had to implement new diffeomorphic techniques for rebalancing of small classes, as we now outline.

\subsubsection{Enrichment by Diffeomorphic Interpolations}\label{s:interpolation}

In automatic classification via machine learning by well known classifiers such as MLPs~\cite{Rosenblatt:1957a}, CNNs, RFs~\cite{Cutler:2012a,Breiman:2001a}, or SVMs~\cite{Cortes:1995a}, strong imbalance between class sizes tends to degrade classification accuracy, specifically among smallest classes. Since our benchmark dataset of MV surfaces was derived from unbalanced classes of patients, we have implemented several diffeomorphic deformation algorithms for the rebalancing of small classes. In most applications of automatic classifiers, all cases are described by feature vectors belonging to a fixed Euclidean space, and  enriching any small class $\mli{CL}$ of cases is often implemented via the SMOTE algorithm (see \cite{Chawla:2002a}), which linearly interpolates between neighboring feature vectors of $\mli{CL}$-cases. But for datasets of discretized smooth 3D-shapes, the SMOTE algorithm is not directly applicable because for intrinsic feature vectors (such as those constructed above) convex combinations of feature vectors produce vectors which are not necessarily associated to any smooth surface. So, to enrich any given finite class $\mli{CL} \subset \mathcal{S}$ of smooth 3D-surfaces, we proceed  by \emph{nonlinear diffeomorphic interpolation} between pairs of surfaces $S, \Sigma \in \mli{CL}$ such that $\operatorname{haus}(S,\Sigma)$ is small enough. Recall that the strongly invariant dissimilarity $\operatorname{haus}$ was defined above by \ref{i:haus}. With the notations of section \secref{s:stronglyinvariant}, replace $S$ and $\Sigma$ by $S''= g_S.S$ and $\Sigma''=g_{\Sigma}.\Sigma$, where $g_S$ and $g_{\Sigma}$ are in the group $G3$. Compute an optimized diffeomorphic deformation flow $F= (F_t)$ such that  $F_1(S'') \approx \Sigma''$. In practical applications below, time is discretized and we select an intermediary value $t \in [0,1]$ between $1/2$ and $3/4$ before adding the new smooth surface $S_t = F_t(S'')$ to the class $\mli{CL}$. See \figref{f:diffeshapeinterp} below.

\begin{figure}
\centering
\includegraphics[width=0.5\textwidth,trim=4cm 4.3cm 4cm 3.3cm,clip]{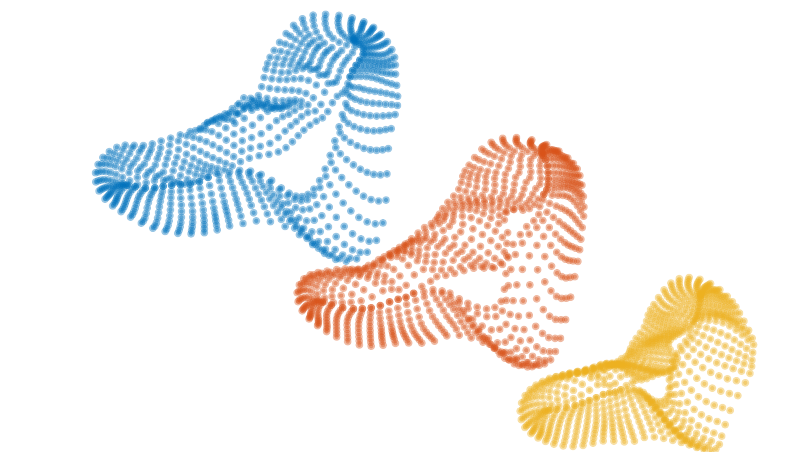}
\caption{The orange surface $S_{3/4}$ is created by diffeomorphic interpolation between the blue surface $S_0 = S$ and the yellow surface $S_1 = \Sigma$.}
\label{f:diffeshapeinterp}
\end{figure}

\subsubsection{Enrichment by Random Diffeomorphic Deformations}\label{s:randomdiffeo}

Fix any given finite class $\mli{CL}$ of smooth $3D$-shapes. To enrich $\mli{CL}$, we have also implemented small random diffeomorphic perturbations of the surfaces in $\mli{CL}$. Our algorithms rely on simulating smooth random Gaussian vector fields depending indexed by time $t$ and all points $x \in\mathbb{R}^3$, before integrating them in time.

\paragraph{\it Simulation of Smooth Gaussian Random Vector Fields} Fix any integer $d >0$, which in this paper will only take values $d\in \{1,2,3\}$. A random $\mathbb{R}^d$-valued vector field $W$,  $x \mapsto W(x)$, indexed by all $x \in \mathbb{R}^3$ is formally a set of random vectors $W(x) \in \mathbb{R}^d$ defined on the same probability space $(\Omega,P)$ and such that $W(x,{\omega})$ is a jointly measurable function of $(x,\omega) \in \mathbb{R}^3 \times \Omega$. Such a random  vector field $W$ is called \emph{Gaussian} when for any finite set of points $x(1), \ldots, x(m)$ in $\mathbb{R}^{3m}$, the random vector $(W(x(1)),\ldots,W(x(m))$---which belongs to $\mathbb{R}^{dm}$---has a Gaussian distribution. Since any multi-dimensional Gaussian is characterized by its mean and its covariance matrix, the probability distribution of a Gaussian random vector field is fully determined by two deterministic functions, namely, $x \mapsto u(x)$, $u(x) = E[W(x)] \in \mathbb{R}^d$ and $(x,y) \mapsto \operatorname{ker}(x,y)= \operatorname{Cov}(W(x),W(y))$, where each $d \times d$ matrix $\operatorname{ker}(x,y)$ is positive semi-definite.

For easier simulations of Gaussian random vector fields $W: \mathbb{R}^3 \to \mathbb{R}^3$, we focus only on the case where $E[W(x)] = 0 $ for all $x$ and the $3 \times 3$ covariance kernel $\operatorname{ker}(x,y)$ is a \emph{diagonal matrix}:
\begin{equation}\label{e:k1k2k3}
\operatorname{ker}(x,y) = \operatorname{diag}[\operatorname{ker}_1(x,y), \operatorname{ker}_2(x,y), \operatorname{ker}_3(x,y) ] \;\;\text{for all} \;\; x,y \in \mathbb{R}^3.
\end{equation}

\noindent For each $j=1,2,3$, we fix a scale parameter $s_j >0$ and we define the radial kernel $\operatorname{ker}_j(x,y)$ for all $ x,y \in \mathbb{R}^3$ by
\begin{equation} \label{e:kj}
\operatorname{ker}_j(x,y) = k(x/s_j, y/s_j),
\;\;\text{where}\;\;k(x,y) = \exp(-\|x-y\|^2).
\end{equation}

\noindent In the appendix \secref{s:onedim.gaussian}, we outline our algorithm to numerically simulate a one-dimensional Gaussian random field $x \mapsto U(x)\in \mathbb{R}$ with mean 0 and covariance kernel $k(x,y) = \exp(- \|x-y\|^2)$, indexed by  all $x \in \mathbb{R}^3$. Moreover, $x \mapsto U(x)$ is an almost surely smooth function of $x$. Our simulation algorithm involves the numerical summation over all triplets of non negative integers $(m,n,p)$ of the almost surely converging explicit series
\begin{equation}\label{e:Useries}
U(x) = \sum_{m,n,p} Z_{m,n,p} u_{m,n,p}(x),
\end{equation}

\noindent where each $u_{m,n,p}(x)$ is an explicit deterministic smooth function of $x \in \mathbb{R}^3$, and the $Z_{m,n,p}$ are independent Gaussian random variables having the same mean 0 and standard deviation 1. This numerical summation can be replicated 3 times (with new $Z_{m,n,p}$ each time), to simulate 3 \emph{independent versions} $U^1(x), U^2(x), U^3(x)$ of the smooth Gaussian random field $U(x)$. We then rescale $U^j(x)$ by setting $W^j(x)= U^j(x/s_j)$ for all $x \in \mathbb{R}^3$ and each $j=1,2,3$. Then, for $x \in \mathbb{R}^3$, we define the 3-dimensional smooth Gaussian random field $W(x) = [W^1(x), W^2(x), W^3(x)]$. Moreover, $W(x)$ has mean 0 and diagonal covariance kernel $\operatorname{ker}(x,y)$ given by equations~\eqref{e:k1k2k3} and~\eqref{e:kj}. In \figref{f:2d-rand-field} we display an example of smooth 2-dimensional Gaussian random field $[W^1(x), W^2(x)]$ simulated by numerical summation of the series \eqref{e:Useries}.

\begin{figure}
\centering
\includegraphics[width=0.4\textwidth]{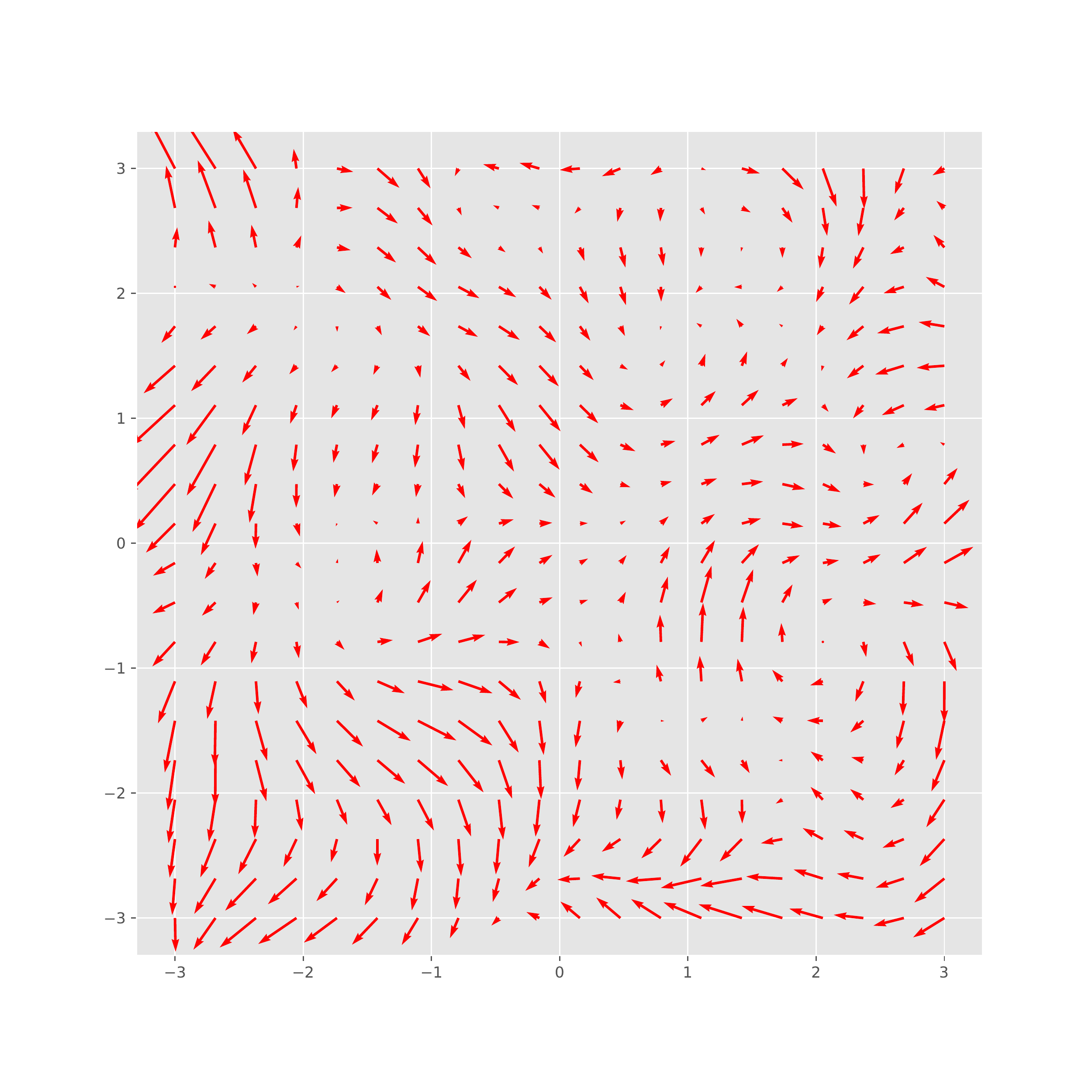}
\caption{An example of simulated 2-dimensional smooth Gaussian random field $W(x)=[W^1(x),W^2(x)]$. For better visualization, all vectors $W(x)$ are normalized to have the same length.}
\label{f:2d-rand-field}
\end{figure}

\paragraph{\it Small Random Diffeomorphic Perturbations of Smooth Surfaces} Let $W(x) \in \mathbb{R}^3$ be the just described smooth Gaussian random vector field indexed by $x \in \mathbb{R}^3$. Let $L$ be any fixed $3 \times 3$ matrix. Then, for $x \in \mathbb{R}^3$ the affine function $x \mapsto Lx$ defines a deterministic vector field, and
\[
V_t(x)= t Lx + \sqrt{t} W(x) \;\; \text{for all}\;\; t \geq 0, x \in \mathbb{R}^3
\]

\noindent defines a time indexed flow $V_t(x)$ of smooth Gaussian random vector fields with deterministic mean vector field $E[V_t(x)] = t Lx$ and diagonal covariance kernel $\operatorname{ker}(x,y)$ given by~\eqref{e:k1k2k3} and~\eqref{e:kj}. We have outlined above how to simulate $x \mapsto W(x)$, which then directly provides the values of $V_t(x)$ for all $(t,x)$. We then numerically generate a stochastic flow $F_t$ of random $\mathbb{R}^3$-diffeomorphisms by \emph{pathwise} discrete integration in $t$ of the stochastic ODE $\mathrm{d}F_t(x)/\mathrm{d}t = V_t(F_t(x))$, with $F_0= \text{id}_{\mathbb{R}^3}$. Time is discretized by fixing a moderate number of instants $t_j = j \delta$, $j=1,2, \ldots$, with small time step $\delta > 0$. Since the random vector fields $V_t(x)$ are almost surely smooth in $(t,x)$, discretized integration in time and space is mathematically stable when discretization meshes tend to 0. To randomly perturb a discretized surface $S \in \mli{CL}$, numerical ODE integration is done separately for each initial $x \in S$. This will generate the points $y_j = F_{t_j}(x)$. One can then enrich the set $\mli{CL}$ by adding the (discretized) smooth surface $F_t(S)$ to $\mli{CL}$ as a virtual case, after checking that $F_t(S)$ is still close enough to $S$. Indeed, for bona fide enrichment of a given class of surfaces one needs to stop the random diffeomorphic deformations $F_t(S)$ at moderate values of time $t$, as displayed for instance in \figref{f:illustration-deformations}.

\begin{figure}
\centering
\includegraphics[width=\textwidth]{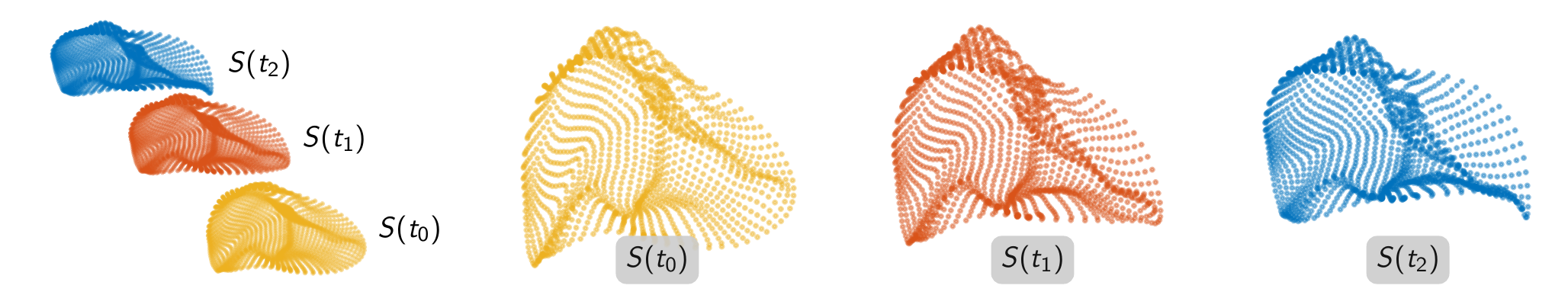}
\caption{Enrichment of shape classes through random diffeomorphic deformations $F_t(S)$. We display the initial surface $S(t_0)$ along with deformed surfaces $S(t_1)$ and $S(t_2)$  at times $0 \leq t_0 < t_1 < t_2$ in an overlay view on the left. The surfaces $S(t_1)$ and $S(t_2)$ are generated by applying a time indexed flow of random diffeomorphic deformations to the initial surface $S(t_0)$. Enlarged views of the individual surfaces $S(t_0)$, $S(t_1)$, and $S(t_2)$ are shown to the right. The applied deformations remain moderate as long as $t_2 - t_0$ (and by that $t_1 - t_0$) is small.}\label{f:illustration-deformations}
\end{figure}

In our application, to enrichment for small classes of discretized MV surfaces, we control random perturbations via simple geometric criteria such as triangulation homogeneity and the trimmed Hausdorff distance.

\section{Setup, Experiments and Results}\label{s:results}

\subsection{Automatic Classification of MVs: Regurgitation vs. Normal}\label{s:results-classification}

We have developed a methodology for automatic classification of smooth 3D-surfaces based on diffeomorphic registration of surfaces. We have implemented this automatic classification methodology for a subset of our MV surfaces, namely,  the union of two disjoint classes of MVs---normal cases vs. regurgitation cases.

\subsubsection{Benchmark Classification Task}

Our original dataset involves 3D-views of human MVs and is acquired by echocardiography for 150 patients, with roughly twenty-five 3D-views per patient, spanning one heart cycle. Each MV 3D-view is discretized by a grid of 1\,600 points in $\mathbb{R}^3$, with 800 points per MV leaflet. \emph{MV Regurgitation} occurs when at midsystole the two MV leaflets AL and PL do not close properly around the coaptation line (instead of tightly closing as in normal patients). At midsystole, this incomplete MV closure leaves a narrow gap between AL and PL, inducing a blood flow leak in the wrong direction during systole, which weakens the normal blood flow through arteries. MV regurgitation is not rare after age 60, and severe regurgitation requires surgical MV repair.

Our benchmark application is ML classification of regurgitation versus normal cases, based on the diffeomorphic matching techniques introduced above. In 3D-echocardiography, MV regurgitation is best visible at midsystole, so we kept only the MV views acquired at times closest to midsystole, namely four views per regurgitation case and two views per normal case. This defines an initial benchmark dataset $\mathcal{D}$ of 3D-discretized MV surfaces, partitioned into 120 regurgitation cases and 200 normal cases.

\subsubsection{Enrichment of Benchmark Data Set}

Recall that the strongly invariant dissimilarity $\operatorname{haus}(S,\Sigma)$ defined in \ref{i:haus}. As outlined in \secref{s:interpolation}, we enrich the initial class of 120 regurgitation cases by diffeomorphic interpolation between pairs $S,\Sigma$ of regurgitation 3D-views having small dissimilarity $\operatorname{haus}(S,\Sigma)$. Our shape interpolation technique is thus applied to 80 pairs $S,\Sigma$ of regurgitation MV surfaces, and generates 80 new virtual regurgitation cases. After this first enrichment of our benchmark dataset, the two classes ``regurgitation''  and ``normal'' now have the same size of 200.

As outlined in \secref{s:randomdiffeo}, we have then implemented one  small random diffeomorphic deformation for  each one  of these 400 smooth surfaces, in order to generate 200 new virtual ``regurgitation'' cases and 200 new virtual ``normal'' cases. After this second enrichment, our new benchmark  dataset, still denoted $\mathcal{D}$ for simplicity, involves now a total of 800 smooth surfaces, namely 400 regurgitation cases and 400 normal cases.

\subsubsection{3D-Image Cropping of MV Surfaces}

For expert cardiologists inspecting sequences of live echocardiography data in 3D, diagnosis of MV regurgitation includes visually checking if around midsystole small gaps emerge along the coaptation line. We hence deliberately focus our MV-surfaces analysis on an MV area close to the coaptation line, and we restrict our diffeomorphic shape registrations techniques to \emph{cropped 3D-snapshots} of MV surfaces. Our  image cropping  keeps only the central half of each original discretized MV-surface snapshot $S$, namely the 800 grid points of $S$ which are roughly closest to its coaptation line. We display a cropped MV surface in \figref{f:half-mv}. After enrichment and image cropping, our final  benchmark  dataset, which we still denote $\mathcal{D}$, contains a total of 800 cropped smooth MV-surfaces, namely 400 regurgitation cases and 400 normal cases.

\begin{figure}
\centering
\includegraphics[width=0.4\textwidth]{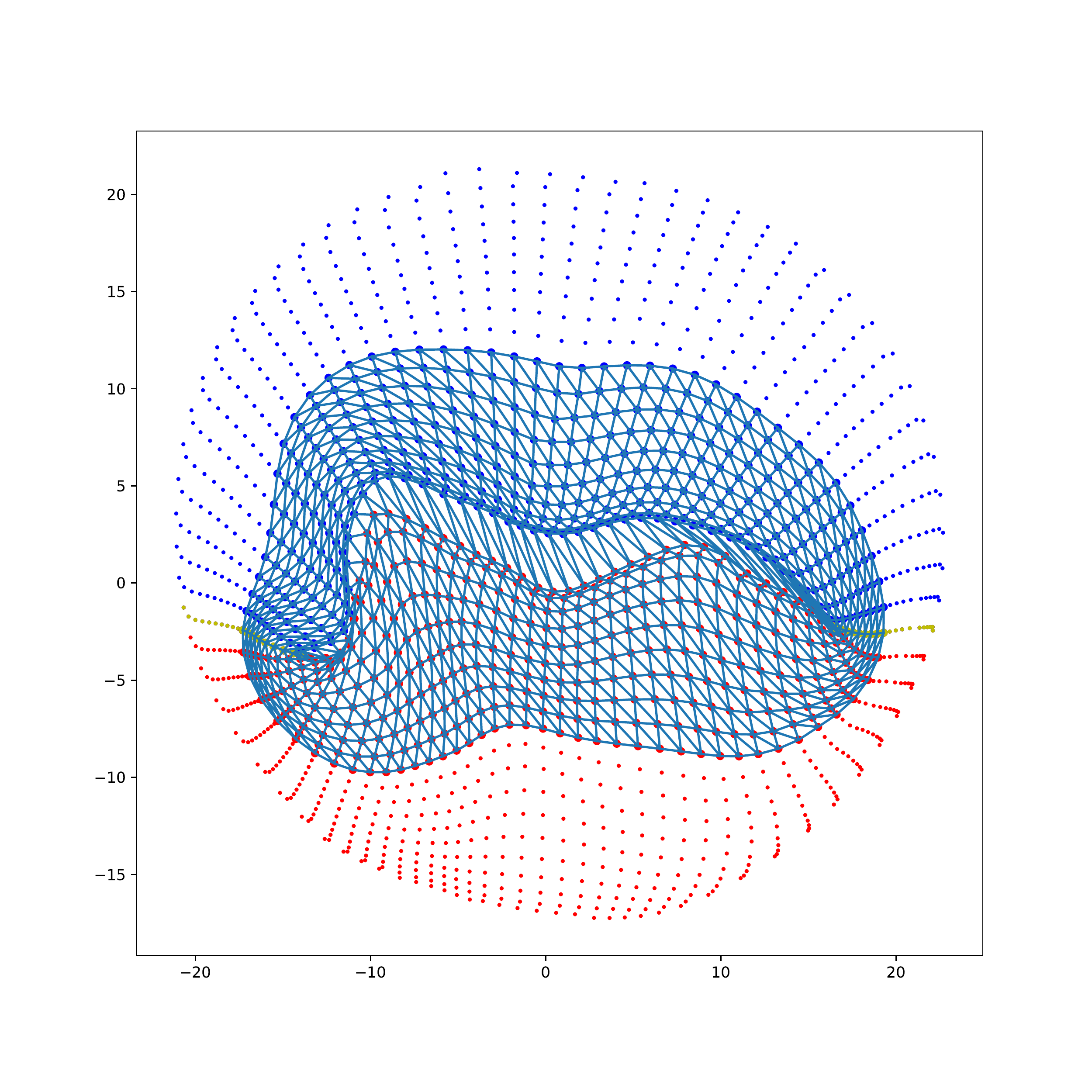}
\caption{Cropped MV surface snapshot for a severe MV regurgitation case: Only the 800 points closest to the coaptation line have been kept.}
\label{f:half-mv}
\end{figure}

\subsubsection{Comparative Statistical Analysis of Kinetic Energy and Strain Quantiles}\label{s:histograms}

Let $G3$ be as above the group of linear transformations of $\mathbb{R}^3$ generated by rotations, translations, and homotheties. For our enriched benchmark dataset $\mathcal{D}$ of 800 cropped MV-surfaces, we have studied the histograms of dissimilarity values taken by the following five strongly invariant dissimilarities:
\begin{enumerate}
\item the kinetic energy $\operatorname{kin}(S,\Sigma)$.
\item the four strain intensity quantiles $\operatorname{quant}_{\alpha}(S,\Sigma)$ with $\alpha \in\{0.05, 0.50, 0.95, 0.99\}$.
\end{enumerate}

\noindent The construction of these dissimilarities was outlined in \secref{s:families} and \secref{s:stronglyinvariant}. In particular, for all $\kappa_1, \kappa_2 \in G3$, these dissimilarities remain unchanged when $S, \Sigma$ are replaced by $\kappa_1.S, \kappa_2.\Sigma$. For fixed $S,\Sigma$, the actual computation of these five dissimilarities requires the numerical diffeomorphic registration of the two cropped MV surfaces $S''=\rho_S.S$ and $\Sigma''=\rho_{\Sigma}.\Sigma$ geometrically derived from $S,\Sigma$ by specific $\rho_S, \rho_{\Sigma} \in G3$, as indicated in \secref{s:stronglyinvariant}.

To study empirically the distributions of these five dissimilarities, we have picked in $\mathcal{D}$ a random subset of 100 normal cases and 100 regurgitation cases, and implemented diffeomorphic registrations for three sets of $10^4$ pairs $(S,\Sigma)$, namely the three sets Nor/Reg, Reg/Reg, and Nor/Nor, corresponding to $\{$Normal vs. Regurgitation$\}$, $\{$Regurgitation vs. Regurgitation$\}$, and $\{$Normal vs. Normal$\}$, respectively. Within each one of these three sets, we have separately computed the histograms of our five dissimilarities \[\{\operatorname{kin}, \operatorname{quant}_{0.05}, \operatorname{quant}_{0.50}, \operatorname{quant}_{0.95}, \operatorname{quant}_{0.99}\}.\] See \figref{f:comp-all-dist}, where we display five figures in each of them three histograms. Here, the blue, red, and green bins correspond to histograms computed within the sets Nor/Nor, Reg/Reg, and Nor/Reg, respectively.

For the pairs of Nor/Nor cases, the blue histograms show that our five dissimilarities are mostly concentrated around small values, indicating a rather tight grouping of normal cases in ``dissimilarity space.'' For the pairs of Reg/Reg cases, the red histograms exhibit roughly larger values of our five dissimilarities, revealing a definitely looser grouping of regurgitation cases in ``dissimilarity space.'' But for the pairs of Nor/Reg cases the green histograms show that all five dissimilarities exhibit fairly high values, and hence indicate a potentially good separability in ``dissimilarity space.'' Moreover, the green histograms of quantiles dissimilarities $\operatorname{quant}_{\alpha}$  observed for pairs of Nor/Reg cases exhibit a marked increase in values when the percentile $\alpha$ increases from $0.05$ to $0.99$. This points to higher discriminating power of $\operatorname{quant}_\alpha$ for higher percentiles $\alpha$.

\begin{figure}
\centering
\includegraphics[width=\textwidth]{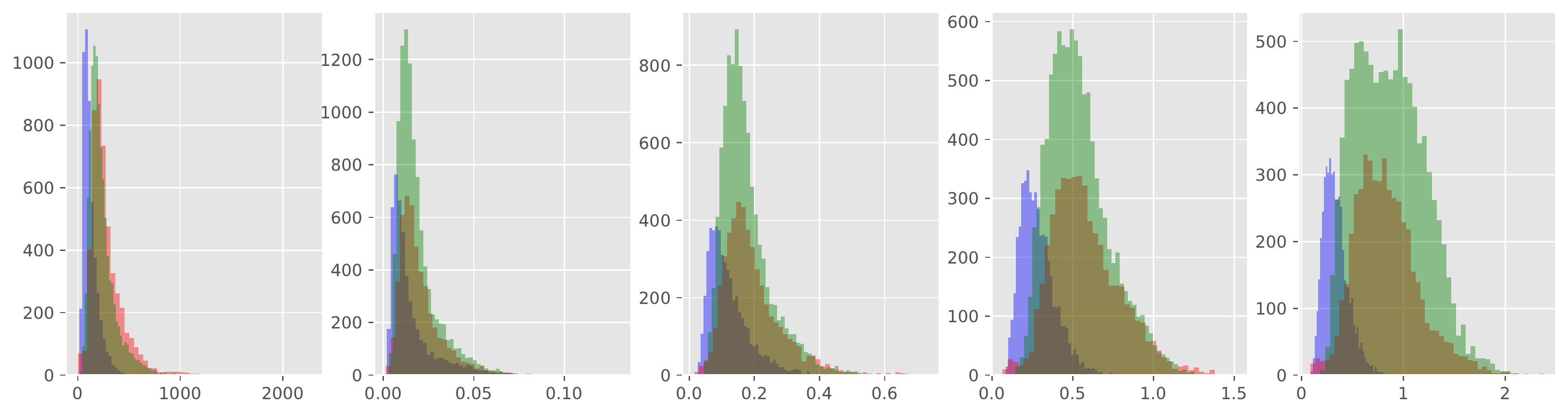}
\caption{The blue, red, and green bars correspond to the three sets Nor/Nor, Reg/Reg, Nor/Reg, respectively, of $10^4$ pairs $(S,\Sigma)$ in $\mathcal{D}$. The histograms of our five strongly invariant dissimilarities $\{\operatorname{kin}, \operatorname{quant}_{0.05}, \operatorname{quant}_{0.50}, \operatorname{quant}_{0.95}, \operatorname{quant}_{0.99}\}$ are shown from left to right. The y-axis represents the number of pairs}
\label{f:comp-all-dist}
\end{figure}

\subsubsection{Selection of Dissimilarities to Improve Discrimination}

Each cropped MV surface $S \in \mathcal{D}$ is discretized by 800 grid points $x(j) \in \mathbb{R}^3$. These points are precisely ordered in successive \emph{nested} rings surrounding the coaptation line. The last three rings $\mli{ring}_3, \mli{ring}_2, \mli{ring}_1$, contain  80 grid points each and are closer and closer  to the coaptation line. After diffeomorphic registration $F$ of $S$ onto another cropped MV snapshot $\Sigma$, the isotropic strain intensities $\text{isi}_x$ derived from $F$ at each point $x \in S$ define an \emph{isotropic strain intensity} vector $\mli{ISI}(S,\Sigma)$ of dimension 800, with coordinates  $\mli{ISI}_j(S,\Sigma)= \mli{isi}_{x(j)}]$, where  $j=1, \ldots, 800$. For $k=80,160,240,800$,  denote $\mli{lastISI}(k)$ the last $k$ strain intensities listed in $\mli{ISI}(S,\Sigma)$, so that $\mli{lastISI}(k)$ correspond to the strain values observed on $(\mli{ring}_1)$, $(\mli{ring}_1 + \mli{ring}_2)$, $(\mli{ring}_1+\mli{ring}_2+\mli{ring}_3)$, $(\text{the whole cropped surface}\; S)$, respectively.

When  $S$ is a regurgitation case and $\Sigma$ is a normal case, as shown in \figref{f:strian-vec}, the presence of gaps along the coaptation line of $\Sigma$  forces  the strain intensities $\mli{isi}_{x(j)}$ when the points $x(j)$ become closer to the coaptation line of $S$. This remark led us to focus on eight quantile dissimilarities, defined by
\begin{align} \label{med.top}
\mli{medstrain}_k(S,\Sigma)   & = 50\% \text{ quantile of } \mli{lastISI}(k), \\
\mli{highstrain}_k(S,\Sigma)  & = 95\% \text{ quantile of } \mli{lastISI}(k),
\end{align}

\noindent for  $k=80,160,240, 800$, respectively.

\begin{figure}
\centering
\includegraphics[width=.6\textwidth]{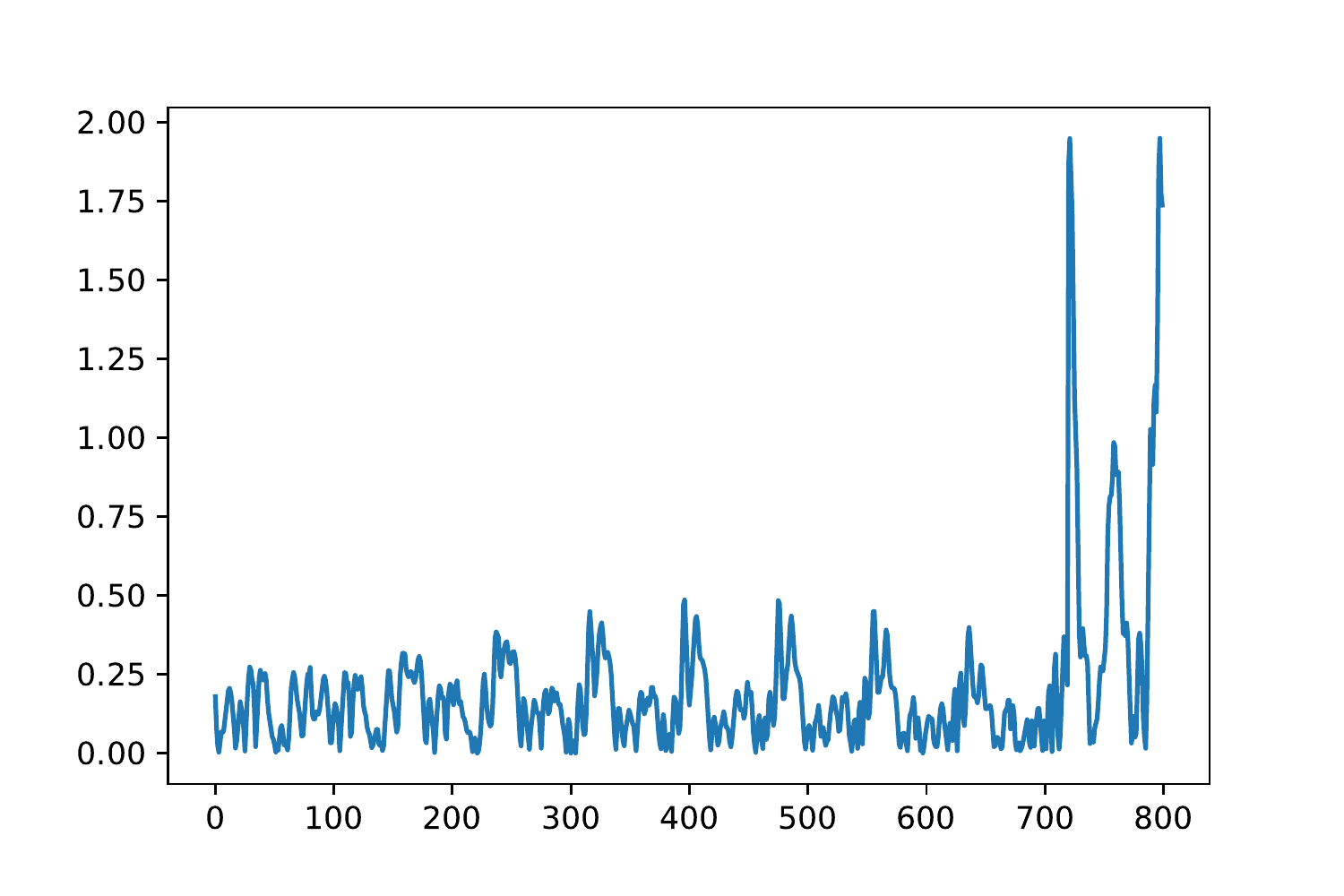}
\caption{An example of strain vector $\mli{ISI}(S,\Sigma)$ of dimension 800, where the cropped surfaces $S$ and $\Sigma$ are a regurgitation case and a normal case, respectively. The last 100 strain intensities correspond to the 100 points of $S$ closest to the coaptation line and clearly tend to have higher values.}
\label{f:strian-vec}
\end{figure}

Our  choice of 50\% and 95\% quantiles was motivated by the histograms analysis outlined in \secref{s:histograms}. For better interpretability, we have also replaced the strongly invariant kinetic energy dissimilarity $\operatorname{kin}(S,\Sigma)$ by its square root $\operatorname{sqrtkin}(S,\Sigma) = \sqrt{\operatorname{kin}(S,\Sigma)}$, which is a bona fide distance between smooth surfaces. Overall, from now on we consider only the nine strongly invariant dissimilarities
\begin{equation}\label{e:listdissim}
\begin{array}{lll}
D_1 =\mli{highstrain}_{80}, \quad&  D_2 = \mli{highstrain}_{160}, \quad& D_3 =\mli{highstrain}_{240},\\
D_4 =\mli{highstrain}_{800}, \quad& D_5 =\mli{medstrain}_{80},\quad& D_6 =\mli{medstrain}_{160}, \\
D_7 =\mli{medstrain}_{240}, \quad&D_8 = \mli{medstrain}_{800},\quad& D_9 =\mli{sqrtkin}.
\end{array}
\end{equation}

\subsubsection{Intrinsic Feature Vectors for Automatic Classification of MV Surfaces}

To implement automatic classification of ``regurgitation'' versus ``normal'' MV surfaces within the enriched dataset $\mathcal{D}$ of 800 cropped MV surfaces, we first construct intrinsic feature vectors based on strongly invariant dissimilarities. To this end, we apply the generic approach outlined in section \secref{s:intrinsic}.  Concretely, this involves two key steps:
\begin{enumerate*}
\item Select and fix a \emph{reference set} $\mli{REF} = \{\Sigma_1, \ldots, \Sigma_r\}$ of $r$ MV surfaces.
\item Fix the set of nine strongly invariant dissimilarities $\mli{DIS} = \left[D_1,D_2,\ldots,D_9\right]$ listed in~\eqref{e:listdissim}.
\end{enumerate*}

Each cropped MV surface $S$ in $\mathcal{D}$ will be described by the group $\operatorname{vec}(S)$ of  $9\times r$ \emph{intrinsic features}  defined by $\operatorname{vec}(S)_{i,j} = D_i(S,\Sigma_j)$, where $i =1, \ldots, 9$, and $j = 1, \ldots, r$. Our histogram analysis comparing several strongly invariant dissimilarities indicates that the dissimilarities $D_i(S,\Sigma)$ tend to be higher when the two MV surfaces $(S,\Sigma)$ are in different classes as compared to when $(S,\Sigma)$ are in the same class. This qualitative result indicates that for each $S$ in $\mathcal{D}$, the dissimilarities between $S$ and all the normal MV surfaces should play a key part to classify $S$ correctly by positively contributing to discrimination between ``regurgitation'' and ``normal.'' So this led us to select a reference set $\mli{REF}$ of $r=100$ normal MV surfaces randomly extracted from our set of all 400 cropped normal MV surfaces. The group of intrinsic feature vectors $\operatorname{vec}(S)$ then involves $900$ features to describe each surface $S$. We have explored other choices for the reference set $\mli{REF}$, as indicated further on. We now present our choice of ML classifiers.

\subsection{RF Classification}\label{s:results-rfclassification}

Among ML classifiers, RFs  have been applied widely, with quite convincing performance. Generated by simultaneous training of large sets of decision trees, RFs were introduced by~\cite{Ho:1995a} and popularized by~\cite{Breiman:2001a}, for instance, as well as by the fast emergence of efficient RF software. In RF training, each decision tree is trained on a randomly selected training set, and each entropy optimizing split of a tree node is based on a set of features randomly selected for each node. After training, the RF classifier combines the class predictions generated by each tree, usually by majority voting. For our dataset of 800 cropped MV surfaces, with each surface $S$ described by a group $\operatorname{vec}(S)$ of $900$ features as just outlined, we have automatically trained distinct RF classifiers, using the open source \texttt{rfpimp} software package~\cite{rfpimp-git,rfpimp} which offers flexible tools dedicated to  evaluating the importance of any given subgroup of features by randomly scrambling their values.

\subsubsection{Meta-Parameters of RF Classifiers}

Several well known meta-parameters have to be specified for RF training. After empirical exploration of potential choices, we have selected and fixed the following RF meta-parameters:
\begin{itemize}
\item {\bf Number of Trees:} The number of trees is set to 300. For each tree training, the random training set has size $\frac{2}{3} 800$. Class weights are used to compensate the imbalance of the random training set.
\item {\bf Node Splitting:} The node splitting is based on $30 = \sqrt{900}$ randomly selected features per node. Node impurity is quantified by its Gini index, which is given by $\alpha (1-\alpha)$, where $\alpha$ denotes the frequency of regurgitation cases per node. Node splitting is accepted only if the splitting decreases node impurity by at least $0.002$.
\end{itemize}

\subsubsection{RF Performance Evaluation by OOB Accuracy}

With the preceding specifications, for each fixed case $S$ in $\mathcal{D}$, there is a random  set $B(S)$ of trees whose training set does not include $S$. Here, the average size of $B(S)$ is roughly $37.4\%$ of $300$, i.e., $112$ trees. After training, each tree $\mli{TR}_i \in B(S)$ computes its own prediction $\mli{pred}_i(S)$ for the true  class $\mli{trueC}(S)$ of $S$; the out-of-bag ({\bf OOB}) prediction for $\mli{trueC}(S)$ is then the class $\hat{C}(S)$, which occurs most often among all the $\mli{pred}_i(S)$. The OOB accuracy of the RF classifier is the frequency of correct answers $\hat{C}(S) = \mli{trueC}(S)$ over all cases $S \in \mathcal{S}$. OOB accuracy is known to be a fairly robust estimator for the generalization capacity of the trained RF classifier (see \cite{Breiman:1996a}).

\subsubsection{OOB Accuracy Results for Benchmark Dataset}

For our dataset of 800 cropped MV surfaces, equally split between regurgitation and normal cases, each case is described by the 900 intrinsic features described above. Training a single  RF classifier having 300 trees and computing its OOB accuracy is actually quite fast. We have repeated this operation 1\,000 times, which provided 1\,000 \emph{distinct} RF classifiers, since RF training is a highly stochastic procedure. The OOB accuracies of these 1\,000 RF classifiers range between $96.1 \%$ and $97.6\%$. We display the histogram of OOB accuracies for these 1\,000 RF classifiers in \figref{f:oob-accuracies}.

\begin{figure}
\includegraphics[width=0.5\textwidth]{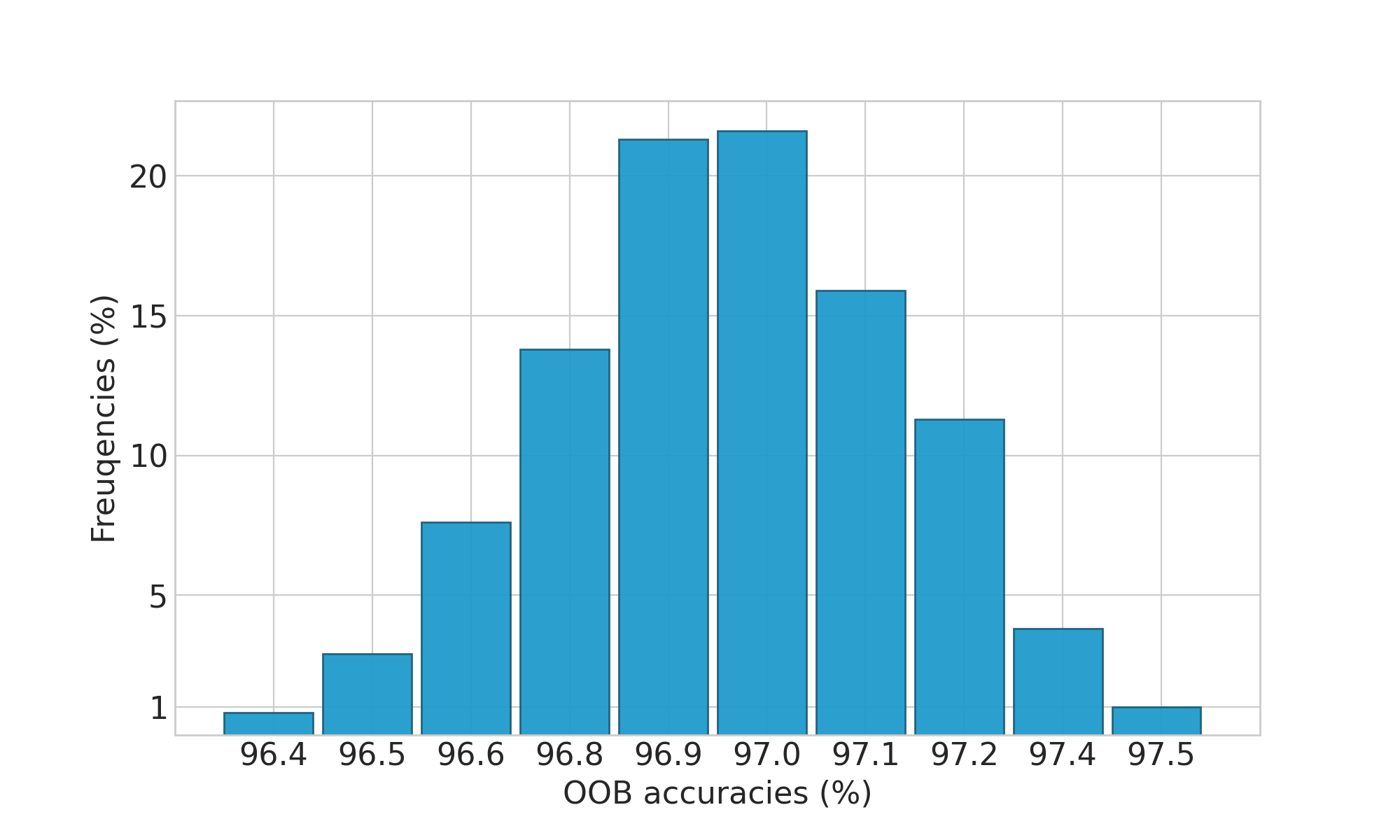}
\caption{For  our benchmark dataset of 800  MV surfaces, we study automatically classification of   ``regurgitation'' MV surfaces versus ``Normal'' MV surfaces. We have derived 900 G3-invariant feature vectors computed via diffeomorphic registration between pairs of surfaces. We have separately trained 1\,000  distinct RF classifiers. The histogram of their OOB accuracies is displayed here, with high  OOB accuracies ranging from  96.1\% to 97.6\%.\label{f:oob-accuracies}}
\end{figure}

The observed generality of RF classifiers with quite high OOB accuracies indicates that our intrinsic features based on strongly invariant dissimilarities between surfaces are quite efficient. The best of our 1\,000 RF classifiers, which we now denote $\mli{RF}^*$, has global OOB accuracy of $97.6\%$. The $2\times 2$ confusion matrix $\mli{CONF}$ of $\mli{RF}^*$ is given by
\[
\mli{CONF} =
\begin{bmatrix}
97.25\% & 2.75\% \\
2.25\% & 97.75\%
\end{bmatrix},
\]

\noindent where $\mli{CONF}_{1,1} = 97.25\% $ and $\mli{CONF}_{2,2} = 97.75\% $ are  the OOB percentages of correctly classified cases, among  regurgitation and normal cases, respectively. As is well known, one can further improve classification accuracy by for instance ``bagging'' the five  best RF classifiers via a standard majority vote. But we have preferred to study more precisely the geometric dissimilarities, which have the strongest impact on $RF^*$.

\subsubsection{Importance Evaluation for Key Groups of Features}

Recall that we had selected (see~\eqref{e:listdissim}) a set of nine strongly invariant dissimilarities $D_1,\ldots, D_9$ computable by diffeomorphic matching between pairs of surfaces $(S1,S2)$, namely, $\mli{highstrain}_{80}$, $\mli{highstrain}_{160}$, $\mli{highstrain}_{240}$, $\mli{highstrain}_{800}$, $\mli{medstrain}_{80}$, $\mli{medstrain}_{160}$, $\mli{medstrain}_{240}$, $\mli{medstrain}_{800}$, $\mli{sqrtkin}$.\footnote{The first four $D_i$ are the $95\%$-quantiles of strain intensities observed on four increasingly larger neighborhoods of the coaptation line on our MV surfaces; the next four $D_i$ are the corresponding $50\%$-quantiles of strain intensities; the last $D_9$ is the square root of the kinetic energy required for optimal diffeomorphic matching of $(S1,S2)$.} For each $i=1 \ldots 9$, the dissimilarity $D_i$ determines a group $\Gamma_i$ of 100 features  $\operatorname{vec}_{i,j}(S) = D_i(S,\Sigma_j)$, where $ j=1 \ldots 100$. A well known technique to evaluate concretely the \emph{importance} $\mli{IMP}_i$ of the group of features $\Gamma_i$  is to compute the loss in OOB accuracy for each feature in $\Gamma_i$, one scrambles the 800 values $D_i(S,\Sigma_j)$ over all $S \in \mathcal{D}$. To implement this random scrambling, one selects 100 random permutations $\mli{PER}_1,\ldots, \mli{PER}_{100}$ of our dataset of 800 cases, and then one replaces each $D_i(S,\Sigma_j)$ by $D_i(\mli{PER}_j(S),\Sigma_j)$. This massive random scrambling, done for the fixed sub group of features $\Gamma_i$ naturally yields a decrease in OOB accuracy for the already trained classifier $\mli{RF}^*$. For a more precise estimate of this accuracy decrease, the scrambling operation is repeated 100 times for each subgroup $\Gamma_i$ and one computes the average OOB accuracy decrease $\mli{IMP}_{i}$,  which then quantifies  the importance of the features subgroup $\Gamma_i$. The importances $\mli{IMP}_{i}$ of our nine subgroups of features $\Gamma_i$ are reported in decreasing order in \tabref{t:imp}, which, hence, also ranks our nine dissimilarities $D_i$ by decreasing importance for our discrimination task. These results identify the three most important dissimilarities $D_1,D_2,D_9$, with importances $28\%$, $13.5\%$, and $11\%$, respectively. The top two are the $95\%$-quantiles of strain intensities $D_1=\mli{highstrain}_{80}$ and $D_2=\mli{highstrain}_{160}$ focused n the 80 and 160  points closest to the coaptation line, respectively. The third top dissimilarity is the kinetic energy $D_9 = \mli{sqrtkin}$.

\begin{table}
\caption{Importances $\mli{IMP}_{i}$ of our nine subgroups of features $\Gamma_i$.\label{t:imp}}
\begin{tabular}{lr}\toprule
\bf Dissimilarity        & \bf Importance \\
\midrule
$\mli{highstrain}_{80}$  & 28.0\% \\
$\mli{highstrain}_{160}$ & 13.5\% \\
$\mli{sqrtkin}$          & 11.0\% \\
$\mli{highstrain}_{240}$ &  9.0\% \\
$\mli{highstrain}_{800}$ &  8.5\% \\
$\mli{medstrain}_{160}$  &  8.5\% \\
$\mli{medstrain}_{80}$   &  7.5\% \\
$\mli{medstrain}_{800}$  &  7.0\% \\
\bottomrule
\end{tabular}
\end{table}

The preceding importance analysis led us to also implement another RF classification based on only on the three most important subgroups of intrinsic features, namely, the three groups of 100 features each defined by the dissimilarities $D_1=\mli{highstrain}_{80}, D_2=\mli{highstrain}_{160}$ and the kinetic energy $D_9= \mli{sqrtkin}$. We trained 100 new RF classifiers based only on these three groups of 100 intrinsic features, and the best new $\mli{RF}^{**}$ classifier based on this reduced set of  300 intrinsic features  reached a global OOB accuracy of $96.7\%$, which is very close to the $97.6\%$ OOB accuracy of the best classifier $\mli{RF}^*$ based on 9 dissimilarities and the associated 900 intrinsic features. Thus, to discriminate efficiently between regurgitation and normal MVs, the strongest informations derived from dissimilarities based on diffeomorphic registration are clearly provided by the $95\%$-quantiles of the strain values observed very close to the coaptation line, and are efficiently completed by the kinetic energy of diffeomorphic registration between pairs of MV surfaces.

\subsection{Computing Times}\label{s:computing-times}

Most of the computing time involved in the preceding benchmark application was consumed by implementation of diffeomorphic registration for about $800(100) = 80\,000$ pairs of cropped MV surfaces, which took about 80 hours of computing time on 80 nodes of the Opuntia cluster of the Computing Center at the University of Houston, as detailed below. After all diffeomorphic registrations were completed, automatic training of any RF classifiers was quite fast, with a runtime inferior to 30 seconds for each such training. This short RF training time was to be expected since we had only 800 cases and 300 trees.

Recall that our cropped MV surfaces were all discretized by 800 points triangulated grids. On a laptop with 1.4 GHz Quad-Core Intel Core i5, 16 GB memory, 2133 MHz LPDDR3, the diffeomorphic registration of one pair of MV surfaces using our solver requires an average computing time of 20 seconds. The large number of such diffeomorphic registrations needed here were naturally distributed by deploying our solver on a cluster. For our benchmark dataset of 800 cropped MV surfaces, our solver had to be executed for 80\,000 pairs of MV surfaces, requiring about 2\,493\,min for parallel computing on 80 nodes used at full availability. Since at any given time we only had simultaneous use of about 40 nodes, completing these 80\,000 diffeomorphic registrations took about 80 hours of computing time.

\section{Conclusions and Future Work}\label{s:conclusions}
Given any dataset $\mathcal{D}$ of smooth 3D-surfaces partitioned into a finite set of disjoint classes, implementing automatic class prediction by supervised ML classifiers such as RFs~\cite{Cutler:2012a,Breiman:2001a}, MLPs~\cite{Rosenblatt:1957a}, or SVMs~\cite{Cortes:1995a} usually requires the characterization of every  surface $S \in \mathcal{D}$ by a computable feature vector $\operatorname{vec}(S) \in \mathbb{R}^N$, for some fixed $N$. In this paper,  we use diffeomorphic registration of smooth surfaces to develop several algorithms dedicated to efficient implementation of ML for automatic classification of smooth 3D-surfaces. In earlier research, we have developed a software solver to compute a diffeomorphic registration $f: S \to S'$ for any given pair $(S,S')$ of smooth 3D-surfaces, with  $f$ embedded in a smooth diffeomorphic flow with minimal kinetic energy $\operatorname{kin}(S,S')$~\cite{Zhang:2021a,Azencott:2010a}. We also compute the isotropic strain of $f$ at each $x \in S$, and several  quantiles $\operatorname{quant}_{\alpha}(S,S')$ of all these  strain values. From  each $\operatorname{quant}_{\alpha}(S,S')$, as well as from $\operatorname{kin}(S,S')$, we derive a $G3$-invariant dissimilarity $\operatorname{dis}(S,S')$ between surfaces $S$ and $S'$. Here, $G3$ is the group of  mappings from $\mathbb{R}^3$ to $\mathbb{R}^3$, generated by translations, rotations, and homotheties. We have then  outlined  how this family of dissimilarities $\operatorname{dis}(S,S')$ can generate natural feature vectors $\operatorname{vec}(S) \in \mathbb{R}^N$ defined for all surfaces $S \in \mathcal{D}$, and invariant by the group $G3$.  This is an efficient first step to apply ML for supervised automatic classification in $\mathcal{D}$.

Moreover, since  unbalanced  class sizes tend to degrade classification accuracy by standard ML approaches, we develop two class enrichment algorithms derived from optimized diffeomorphic registration. Given two  surfaces $S,S' \in \mathcal{D}$  belonging to the same given class $\mli{CL}$ we need to enrich, numerical diffeomorphic registration of $S$ and $S'$ automatically computes a set of surfaces $S_t \in \mathcal{D}$ depending smoothly on $t \in (0,1)$, which are all diffeomorphic to $S$,  and verify $S_0 = S, S_1 = S'$. Then, if $S,S'$ are sufficiently close to each other, any interpolating surface $S_t$ can be added to class $\mli{CL}$ as a virtual new case. This  interpolation technique enables a powerful extension of the well known SMOTE enrichment technique which is restricted to linear interpolation between euclidean vectors.  For class enrichment in $\mathcal{D}$, we have also developed and implemented small random perturbations of smooth 3D surfaces by applying flows of \emph{random smooth diffeomorphisms} of $\mathbb{R}^3$, generated by time integration of smooth Gaussian random vector fields indexed by $\mathbb{R}^3$. Numerical implementation required a sophisticated convergence analysis for stochastic series expansions of smooth Gaussian vector fields.

We have successfully applied all the preceding tools and methodologies  to  the automatic classification of 800 MV  surfaces into two classes: ``regurgitation'' versus ``normal.'' The $G3$-invariant feature vectors we computed had dimension 900 and were derived from nine distinct dissimilarities (kinetic energy and eight strain quantiles) . The well known RF classifiers (here used with 300 trees) reached a high classification accuracy of $97.6\%$. A comparative importance analysis between 9 groups of features  showed that the 3 most important dissimilarities were the $95\%$-quantiles of isotropic strain observed near the coaptation line of each MV surface and the kinetic energy. In fact, after reduction of our feature vectors to the 300 features associated to the three most important dissimilarities, the RF classification accuracy was still quite high at $96.7\%$.

In ongoing research, we will apply the proposed techniques to automatic classification of much larger datasets of 3D smooth surfaces, into dozens of classes. This will require the development of ML techniques specifically dedicated to reducing the computing times required by large numbers of diffeomorphic registrations.

\noindent\textbf{Acknowledgements:} This work was partly supported by the National Science Foundation under the awards DMS-1854853, DMS-2012825, and DMS-2145845. Any opinions, findings, and conclusions or recommendations expressed herein are those of the authors and do not necessarily reflect the views of NSF. This work was completed in part with resources provided by the Research Computing Data Core at the University of Houston.

\begin{appendix}

\section{Smooth one-dimensional  Gaussian Random Fields on $\mathbb{R}^3$}
\label{s:onedim.gaussian}
For  $t \in \mathbb{R}$, the density function $\frac{1}{\sqrt{\pi}}\exp(-t^2) $ defines a measure $\gamma$ of mass 2. Let $\Gamma = L_2(\mathbb{R}, \gamma)$. The Gaussian kernel $g(t,t') = \exp(-(t-t')^2)$ defines a linear operator $G: \Gamma \to \Gamma$ given by, for  $\phi \in \Gamma$,
\[
G\phi(t) = \int_{\mathbb{R}} g(t,t')\phi(t') \;\mathrm{d}\gamma(t') \;\; \text{for all} \;\;  t \in \mathbb{R}.
\]

\noindent Then $G$ is a Hilbert-Schmidt operator with known eigenvalues $\lambda_n$ and eigenfunctions $\phi_n$, which form an orthonormal basis of $\Gamma$. One has the converging series expansions
\begin{align}
G\phi (t) =\sum_{n \geq 0} \lambda_n \langle \phi_n, \phi \rangle_{\Gamma} \phi_n(t)\;\; \text{for all} \;\; \phi \in \Gamma,\label{e:gphi}\\
g(t,t')=\sum_{n \geq 0}\lambda_n \phi_n(t)\phi_n(t') \;\; \text{for all} \;\; t,t' \in \mathbb{R}.\label{e:expan}
\end{align}

The $\phi_n$ are expressed below in terms of the Hermite polynomials $H_n(t)$, which are recursively given by $H_0(t) = 1$ and
\begin{equation}\label{e:hermite}
H_{n+1}(t) = 2 t H_n(t) -2 n H_{n-1}(t)
\end {equation}

\noindent for all $t \in \mathbb{R}$, $n = 0, 1, \ldots$. Each $H_n(t)$ has degree $n$ and  leading term $(2t)^n$. The $\lambda_n$ and $\phi_n$ are given by the known formulas (see \cite{fasshauer2011positive})
\begin{equation}
\label{e:eigen}
\lambda_n = a^{n + 1/2}
\quad \text{and}\quad
\phi_n(t) = \frac{b}{\sqrt{2^n n!}} \exp(-c t^2) H_n(h t)
\end{equation}

\noindent where $a =1/(1 +\frac{\sqrt{5}}{2}) < 1/2$, $b= 5^{1/8}$, $c= (\sqrt{5}-1)/2$, and $h = 5^{1/4}$. Rewrite the radial kernel $k(x,y) = \exp(- |x-y|^2)$, $x,y \in \mathbb{R}^3$, as follows
\begin{equation} \label{e:kxy}
k(x,y) = g(x_1,y_1) g(x_2,y_2) g(x_3,y_3),
\end{equation}

\noindent where $x,y$ have coordinates $x_j, y_j$. Endow $\mathbb{R}^3$ with the product measure $\theta = \gamma^3$. Define the Hilbert-Schmidt operator $\mli{HS}$ from $L_2(\mathbb{R}^3, \theta)$ into itself by
\[
\mli{HS}\psi(x)=\int_{\mathbb{R}^3} k(x,y)\psi(y)\;\mathrm{d}\theta(y) \;\; \text{for all functions} \; \psi \in L_2(\mathbb{R}^3, \theta).
\]

\noindent The eigenvalues and eigenfunctions of $\mli{HS}$ are clearly given by
\begin{align}
\tau_{m,n,p} & = \lambda_m  \lambda_n  \lambda_p  \label{e:eigenHS} \\
\psi_{m,n,p}(x) & = \phi_m(x_1) \phi_n(x_2) \phi_p(x_3)\label{e:eigenpsi}
\end{align}

\noindent for all integers $m,n,p$ and all $x=[x_1,x_2,x_3] \in \mathbb{R}^3$. The $\psi_{m,n,p}$ provide an orthonormal basis for $L_2(\mathbb{R}^3, \theta)$, and one has the converging series expansions
\begin{align}
\mli{HS}\psi &= \sum_{m,n,p} \tau_{m,n,p} \langle \psi_{m,n,p}, \psi \rangle_{L_2(\mathbb{R}^3, \theta)} \psi_{m,n,p} \;\; \text{for all} \;\; \psi \in L_2(\mathbb{R}^3, \theta),\label{e:Hpsi} \\
k(x,y) &= \sum_{m,n,p} \tau_{m,n,p} \psi_{m,n,p}(x) \psi_{m,n,p}(y) \;\; \text{for all} \;\; x,y \in \mathbb{R}^3. \label{e:expandkxy}
\end{align}

We now concretely construct an explicit stochastic series converging to a  smooth $\mathbb{R}$ valued random Gaussian vector field $U : \mathbb{R}^3 \to \mathbb{R}$ with mean 0 and covariance kernel $k(x,y)$. The construction is summarized in \thmref{thm:convergence}. The speed of convergence for this series is studied in \thmref{thm:convspeed}. The proofs of these two theorems are given below in section \secref{s:proofs}).

\begin{theorem}\label{thm:convergence}
Let $Z_{m,n,p}$ be any sequence of standard independent Gaussian random variables, indexed by the integer triplets $m,n,p$, and defined on the same probability space $(\Omega,P)$. Then $P$-almost surely, the following stochastic series converges pointwise for all $x \in \mathbb{R}^3$ to a finite limit denoted
\begin{equation}\label{e:U}
U_x = \sum_{m,n,p} Z_{m,n,p} \sqrt{\tau_{m,n,p}}\, \psi_{m,n,p}(x)
\end{equation}

The one dimensional random vector field $U_x  \in\mathbb{R}$ defined by this series is Gaussian  with mean 0 and covariance kernel $E(U_x U_y) = k(x,y)$ for all $x,y \in \mathbb{R}^3$. The function $x \mapsto U_x$ is in $L_2(\mathbb{R}^3,\theta)$, and is $P$-almost surely smooth in $x \in \mathbb{R}^3$. Moreover, the series \eqref{e:U} also converges to $U$ in $L_2(\mathbb{R}^3, \theta)$-norm.
\end{theorem}

\begin{theorem}\label{thm:convspeed}
The partial sums $U(N)_x$ of the series $U_x$ are denoted
\begin{equation}\label{e:UN}
U(N)_x =  \sum_{(m,n,p) \in B(N)} Z_{m,n,p} \sqrt{\tau_{m,n,p}} \psi_{m,n,p}(x)
\end{equation}

\noindent where $B(N) = \{m,n,p | \max(m, n, p) \leq N\}$. With  probability $q(N) > 1 - \frac{1}{5 N^{10}}$, one has the uniform bound
\begin{equation}\label{e:remainder}
|U_x - U(N)_x | \leq  c \exp(|x|^2/2) N^{7/10}/2^N
\end{equation}

\noindent for all $x \in \mathbb{R}^3$ and all $N >6$, where $c$ is a universal constant, which does not depend on $N$ nor on $x$.
\end{theorem}

\section{Numerical Implementation of  Random Diffeomorphic Deformations} \label{s:random}

Select independent standard Gaussian random variables $Z^j_{m,n,p}$ indexed by $j=1, 2,3$ and by all the triplets $(m,n,p)$. Then as in \eqref{e:U}, define on $\mathbb{R}^3$ three independent $\mathbb{R}$-valued \emph{smooth} Gaussian random fields $U^j, j=1,2,3$,  by the almost surely convergent series
\begin{equation}\label{e:Ujx}
U^j_x  = \sum_{m,n,p} Z^j_{m,n,p} \sqrt{\tau_{m,n,p}} \psi_{m,n,p}(x) \;\;\text{for all}\; x \in \mathbb{R}^3.
\end{equation}

Fix any positive scale parameters $s_1, s_2, s_3$ and the three scaled Gaussian kernels
\[
k_j(x,y) = k(x/s_j, y/s_j)\;\;\text{for}\;\;j=1,2,3,\;\;\text{and for all}\;\; x,y \in \mathbb{R}^3.
\]

\noindent Define the random vector field $W_x= [W^1_x, W^2_x, W^3_x]$ by
\begin{equation}\label{e:Wjx}
W^j_x =U^j_{x/s_j}\;\;\text{for}\;\;j=1,2,3,\;\;\text{and all}\;\; x\in \mathbb{R}^3.
\end{equation}

\noindent Due to theorem \thmref{thm:convergence} we have
\begin{enumerate*}
\item The random vector field $x \mapsto W_x$, $W_x\in \mathbb{R}^3$, is Gaussian with zero mean and independent coordinates.
\item The covariance kernel of $W$ is the $3 \times 3$ diagonal matrix $\text{Cov}(W_x ,W_y) = \text{diag}[ k_1(x,y), k_2(x,y), k_3(x,y)]$ for all $x,y \in \mathbb{R}^3$.
\item Almost surely, the function $x \mapsto W_x$ is infinitely differentiable in $x \in \mathbb{R}^3$.
\end{enumerate*}

For numerical computations of $W_x$ when $x \in \mathbb{R}^3$ is given, the main point is to determine how many terms to keep in the basic stochastic series \eqref{e:U} defining $U_x$. As seen in \thmref{thm:convergence}, for $N \geq 10$, with probability $q(N)$ practically equal to 1, the partial sums $U(N)_x$ approximate $U_x$ at a speed faster than $c \exp{|x|^2 /2} N^{7/10}/2^N$, where $c$ is a numerical constant. Our explicit theoretical upper bound for $c$ is too large for pragmatic  estimates of the remainder $|U_x - U(N)_x|$. But our numerical experiments and more concrete  estimates of Hermite polynomials (see \cite{Bonan:1990a}) indicate that for $|x| \leq 4$, one can obtain a good approximation of the series  $U_x$ by keeping only the terms $ T_{m,n,p}(x)$ such than $m,n,p \leq 25$. The restriction $|x| \leq 4$ is not a real constraint since after an adequate  $\mathbb{R}^3$-homothety, any bounded surface  $S$ can become a surface $\Sigma$  included in a small ball of radius $4$ around its center of mass, and after implementing a numeric small random  deformation of $\Sigma \to \Sigma(t)$, one can rescale $\Sigma(t)$ to give it a total area is equal to 1.

\section{Proofs of \thmref{thm:convergence} and \thmref{thm:convspeed}}\label{s:proofs}

\begin{proof}
Clearly, $x \to U(N)_x$ is a random smooth function of $x \in \mathbb{R}^3$, and belongs to $L_2(\mathbb{R}^3,\theta)$. Since the $Z_{m,n,p}$ are decorrelated, one has for all $x,y \in \mathbb{R}^3$,
\[
E[U(N)_x U(N)_y] = \sum_{m,n,p \in B(N)}  \tau_{m,n,p} \psi_{m,n,p}(x) \psi_{m,n,p}(y).
\]

In view of the converging series expansion \eqref{e:expandkxy} of $k(x,y)$, this implies
\begin{equation}\label{e:limcov}
k(x,y) = \lim_{N \to \infty} E[ U(N)_x U(N)_y].
\end{equation}

Define the random function $x \to T_{m,n,p}(x)$ for all $x \in \mathbb{R}^3$ by
\begin{equation}\label{e:Tmnp}
T_{m,n,p}(x) = Z_{m,n,p} \sqrt{\tau_{m,n,p}} \psi_{m,n,p}(x).
\end{equation}

The function $T_{m,n,p}$ belongs to the space $L_2(\mathbb{R}^3,\theta)$, where its norm $A_{m,n,p}$ is given by $A_{m,n,p} = |Z_{m,n,p}| \sqrt{\tau_{m,n,p}}$. Thanks to equations \eqref{e:eigenHS}, \eqref{e:eigen} one has
\[
\sum_{m,n,p} E(A_{m,n,p}^2) = \sum_{m,n,p} \tau_{m,n,p} < \sum_{m,n,p} 1/2^{m+n+p} < \infty.
\]

This implies, since the random variables $A_{m,n,p}$ are independent, that $P$-almost surely, the following stochastic series converges to a (random) limit
\[
\sum_{m,n,p} A_{m,n,p} = \sum_{m,n,p} \|T_{m,n,p}\|_{ L_2(\mathbb{R}^3,\theta)}
\]

\noindent But whenever the series of $L_2(\mathbb{R}^3, \theta)$-norms $|| T_{m,n,p}  ||$ converges, then the stochastic series  $\sum_{m,n,p} T_{m,n,p}$ must also converge in $L_2(\mathbb{R}^3, \theta)$-norm to some random function $U \in L_2(\mathbb{R}^3, \theta)$. Hence, $P$-almost surely, we have
\begin{equation}\label{e:UinL2}
U = \lim_{N \to \infty} U(N)
\end{equation}

\noindent for convergence in $L_2(\mathbb{R}^3, \theta)$-norm. For short, we abbreviate $P$-almost surely as $P$-a.s. In what follows, we will also show that $P$-a.s. one also has the \emph{pointwise} convergence
\begin{equation}\label{e:pointwise}
\lim_{N \to \infty} U(N)_x = U_x  \;\; \text{for all} \;\; x \in \mathbb{R}^3.
\end{equation}

For faster exposition, we derive right away the main consequence of this $P$-a.s. pointwise convergence. In $\mathbb{R}^3$ fix any finite set of points $x(j)$, $j = 1, \ldots, m$. Define the random vectors $X(N)$ and $Y$ in $\mathbb{R}^m$ by their coordinates
\begin{equation}\label{e:XNV}
X(N)_j= U(N)_{x(j)}
\quad \text{and} \quad
Y_j= U_{x(j)}.
\end{equation}

Denote $f_N(z) = E(\exp(i \langle z, X(N) \rangle))$ and $f(z) = E(\exp(i \langle z, Y \rangle))$ the  characteristic functions of $X(N)$ and $Y$ for all $z \in \mathbb{R}^m$.  The  $X(N)$ are Gaussian with mean 0, covariance matrix $Q(N)$, and tend $P$-a.s. to $Y$ . Dominated convergence implies $f(z)= \lim_{N \to \infty} f_N(z)$ for all $z$. The formula $\log{f_N(z)}= z*Q(N) z$ then forces $\log{f(z)}= z*Q z$ with $Q= \lim_{N \to \infty} Q(N)$. Hence $Y$ is Gaussian with mean 0 and covariance matrix Q. This proves that the vector field $x \to U_x$ is Gaussian with mean zero and covariance kernel
\[
E(U_x U_y) = \lim_{N \to \infty} E[ U(N)_x U(N)_y ] = k(x,y),
\]

\noindent where the last equality is due to  \eqref{e:limcov}. Since the covariance kernel $k(x,y)$ is infinitely differentiable  in $x$ and $y$, known results on random Gaussian fields (see \cite{stein1999interpolation}) show that one can find a version of the random field $x \mapsto U_x$, which will $P$-a.s. be smooth in $x \in \mathbb{R}^3$.

We still need to prove the $P$-a.s. pointwise convergence stated in \eqref{e:pointwise}. We first derive bounds for the eigenfunctions $\phi_n$ given in \eqref{e:eigen}. The paper \cite{krasikov2004new} provides sharp universal bounds for the Hermite polynomials $H_n(t)$. These bounds show that for all $n \geq 6$ and all $t \in \mathbb{R}$
\begin{equation}\label{e:Hnt}
|H_n(t)| <  9 n^{-1/12} c_n \exp(t^2/2),
\end{equation}

\noindent where
\begin{equation}
\label{e:boundsoncn}
c_n  <
\begin{cases}
\sqrt{2} n^{1/4} [n!/(n/2)!] & \text{if $n$ is even},\\
\sqrt{5} (n-1)^{3/4} [(n-1)!/((n-1)/2)!] & \text{if $n$ is odd}.
\end{cases}
\end{equation}

Recall the Stirling formula, which states that for all $n>1$
\begin{equation}\label{e:stir}
1 < n! / \text{stir}(n) < 1.5
\end{equation}

\noindent with $\text{stir}(n) = (2 \pi n)^{1/2} (n/e)^n$. This formula implies
\begin{equation}
\label{e:boundswithn}
\begin{cases}
n!/ (n/2)! < 1.5  (2 n/e)^{n/2} & \text{if $n$ is even}, \\
(n-1)!/ ((n-1)/2)!)  < 1.5  (2 (n-1)/e)^{(n-1)/2} & \text{if $n$ is odd}.
\end{cases}
\end{equation}

Combining the bounds \eqref{e:boundsoncn} and \eqref{e:boundswithn}, we get for all $t \in \mathbb{R}$ and all $n \geq 6$
\[
| H_n(t) |<
\begin{cases}
 30 n^{1/6} (2 n/e)^{n/2} \exp(t^2/2) & \text{if $n$ is even}, \\
 30 (n-1)^{2/3} (2 (n-1)/e)^{(n-1)/2}\exp(t^2/2) & \text{if $n$ is odd},
\end{cases}
\]

\noindent and hence a fortiori
\begin{equation}\label{e:endHn}
| H_n(t) | < 30 n^{2/3} (2 n/e)^{n/2} \exp(t^2/2)\;\; \text{for all $n \geq 6$, all $t \in \mathbb{R}$.}
\end{equation}

Equation~\eqref{e:eigen} yields two numerical constants $c$ and $h$ such that, for all $n \geq 1$, all $t \in \mathbb{R}$,
\[
|\phi_n(t)|= \frac{b}{\sqrt{2^n n!} \exp(- c t^2) |H_n(h t)|}.
\]

\noindent Moreover, one has $h^2/2 - c = 1/2$, which implies directly, in view of \eqref{e:endHn},
\begin{equation}\label{e:eH}
\exp(- c t^2) |H_n(h t)| < 30 n^{2/3} (2 n/e)^{n/2} \exp(t^2/2) \; \; \text{for all $n \geq 6$, $t \in \mathbb{R}$.}
\end{equation}

From \eqref{e:stir} and the value of $b$ we get
\[
b / \sqrt{2^n n!} <  n^{-1/4}(2n/e)^{-n/2} \; \; \text{for all $n >1$}.
\]

\noindent So, we can finally  bound the eigenfunctions $\phi_n$ by
\begin{equation}\label{e:phibound}
|\phi_n(t)| < 30 \exp(t^2/2) n^{1/2} \; \; \text{for all $n \geq 6, t \in \mathbb{R}$}.
\end{equation}

\noindent From \eqref{e:eigenpsi} and \eqref{e:phibound}, we derive the following bound, valid  for all $m,n,p \geq 6$, all $x=[x_1,x_2,x_3] \in \mathbb{R}^3$,
\begin{equation}\label{e:psibound}
|\psi_{m,n,p}(x)| = |\phi_m(x_1) \phi_n(x_2) \phi_p(x_3)| < 30^3 \exp{|x|^2/2} (mnp)^{1/2}.
\end{equation}

\noindent From \eqref{e:eigen} we get $\lambda_n < (1/2)^n$ and hence
\begin{equation}\label{e:taubound}
\sqrt{\tau_{m,n,p}} < (1/2)^{m+n+p}.
\end{equation}

We now compute simultaneous probabilistic bounds for the $|Z_{m,n,p}|$. Define the following sequence of independent random events
\begin{equation}\label{e:Emnp}
\text{The event}\;\; E(m, n, p)\;\;\text{is realized iff}\;\;\{|Z_{m,n,p}| \leq 5\sqrt{\log(mnp)}\}.
\end{equation}

\begin{lemma}\label{lem:q(N)}
Define the probabilities $q(N)$ by
\begin{equation}\label{e:defpN}
q(N) = P(\text{all events}\; E(m,n,p)\; \text{with}\; mnp > N\;\text{are realized simultaneously}).
\end{equation}

\noindent Then, for each $N \geq 4$ one has
\begin{equation} \label{e:lowqN}
q(N)  > 1 - \frac{1}{5 N^{10}}
\end{equation}
\end{lemma}

\begin{proof}
Any standard Gaussian random variable $Z$, verifies for all $t >1$
\[
\text{Prob}(|Z| > t) < 2 \int_{s>t} \frac{s}{\sqrt{2 \pi}}  \exp(-s^2/2) \,\mathrm{d}s = \sqrt{2 / \pi} \exp(-t^2/2)
\]

\noindent and hence for any integer $r \geq 2$
\begin{equation}\label{e:Z}
P(|Z| > 5 \sqrt{\log(r)} ) < \frac{2}{3 r^{12.5}}.
\end{equation}

Since the $E(m,n,p)$ are independent, definition \eqref{e:defpN} implies
\[
q(N) = \prod_{m,n,p \;|\;mnp > N }  P( E(m,n,p) ).
\]

\noindent Due to \eqref{e:Z} this yields for $N \geq 4$
\[
q(N) > \prod_{m,n,p \,|\,mnp > N }  \left[ 1- \frac{1}{(mnp)^{12.5}} \right].
\]

\noindent For $0 < u < 10^{-3}$ one has $1 - u > \exp(-2u)$ so that for $N \geq 4$
\[
q(N)  > \exp(- 2 s(N))
\quad \text{with} \quad
s(N)=  \sum_{ m,n,p \,|\, mnp > N }  \frac{1}{(mnp)^{12.5}}.
\]

For $N \geq 4$ one has
\begin{equation}\label{e:card}
\begin{aligned}
\text{card}\{m,n,p\,|\,mnp = N\}
& < \sum_{m=1}^N \text{card} \{n,p\,|\,np = N/m\} \\
& < \sum_{m=1}^N N/m < N + N \log(N),
\end{aligned}
\end{equation}

\noindent which implies for $N \geq 6$.
\begin{equation}\label{e:sN}
s(N) \leq \sum_{k > N}  \frac{1}{k^{12.5}}(k + k \log(k))
\leq \sum_{k > N} 1/ k^{11} \leq 1/10 N^{10}
\end{equation}

\noindent Hence, for $N \geq 6$ we obtain
\[
q(N) > \exp(- 2/10 N^{10}) > 1 - \frac{1}{5 N^{10}}
\]
\end{proof}

We now study the series remainders $|U_x-U(N)_x|$ for $N>6$. By definition of $U(N$), we have
\[
|U_x-U(N)_x| < \sum_{(m,n,p) \in G(N) } |T_{m,n,p}(x)|,
\]

\noindent where $G(N) = \{ m,n,p | \max(m,n,p) >N \}$. Denote $\beta(x) = 30^3 \exp(|x|^2/2)$. With probability $q(N) > 1 - \frac{1}{5 N^{10}}$, we will have $|Z_{m,n,p}| \leq 5 \sqrt{\log(mnp)}$ for all $(m,n,p)$ such that $mnp >N$. In view of the two bounds, \eqref{e:psibound} and \eqref{e:taubound}, we conclude that with probability $q(N)$ we will have, for all $x \in \mathbb{R}^3$ and  all $m,n,p$ verifying $(\min(m, n,p)>6) \land (mnp>N)$,
\begin{equation}\label{e:seriesbound}
\begin{aligned}
|T_{m,n,p}(x)|
&<  5 \beta(x) \log(mnp) (mnp)^{1/2} 2^{m+n+p} \\
&< \frac{10}{2^{m+n+p}} \beta(x) (mnp)^{7/10},
\end{aligned}
\end{equation}

\noindent where the 2nd inequality derives from $\log(mnp) < 2 (mnp)^{1/5}$. Since $mnp>N$, whenever $(m,n,p) \in G(N)$ we conclude that with probability $q(N)$, we will have for all $x \in \mathbb{R}^3$, $|U_x-U(N)_x|< 10 \beta(x) r(N)$ with
\[
r(N) = \sum_{(m,n,p)\in G(N)} \frac{1}{2^{m+n+p}} (mnp)^{0.7}.
\]

For $(m,n,p) \in G(N)$, one has $max(m,n,p) > N$, and hence
\[
r(N) < 3 \sum_{(m \leq p ,n \leq p) \land ( p >N )} \frac{1}{2^{m+n+p}}(mnp)^{0.7}.
\]

\noindent Since
\[
\sum_{m \leq p, n \leq p} \frac{1}{2^{m+n}}(mn)^{0.7}
< \left[ \sum_m \frac{1}{2^m} m^{0.7}\right]^2 < 4,
\]

\noindent we obtain
\[
r(N) < 12\sum_{p > N}  \frac{1}{2^p} p^{0.7} < \frac{24}{2^N} N^{0.7} .
\]

Hence, with probability $q(N) > 1 - \frac{1}{5 N^{10}}$, we will have for all $x \in \mathbb{R}^N$,
\[
|U_x-U(N)_x|
< 10 r(N)\beta(x)
< \frac{240}{2^N} N^{0.7}\beta(x)
= 240 (30^4 \exp(|x|^2/2)),
\]

\noindent which proves the announced bound \eqref{e:remainder}. Equation \eqref{e:sN} clearly forces the series
\[
\sum_{m,n,p} P( E(m, n, p) )
\]

\noindent to be finite. Hence, by Borel-Cantelli's lemma there is random integer $\mli{RAND}$ which is $P$-a.s. finite, and such that all the events $E(m,n,p)$ with  $mnp > \mli{RAND}$ are simultaneously realized. The arguments just used above show that the bound~\eqref{e:seriesbound} on $| T_{m,n,p}(x) |$  will hold for $mnp > \mli{RAND}$. Whenever $\mli{RAND}$ is finite, this forces the pointwise convergence of the series $U_x$ for all $x \in \mathbb{R}^3$. We have thus proved that $P$-a.s the series $U_x$ will converge pointwise for all $x \in \mathbb{R}^3$.
\end{proof}
\end{appendix}

\printbibliography

@article{Mang:2015a,
	author = {Mang, A. and Biros, G.},
	journal = {SIAM Journal on Imaging Sciences},
	number = {2},
	pages = {1030--1069},
	title = {An inexact {N}ewton--{K}rylov algorithm for constrained diffeomorphic image registration},
	volume = {8},
	year = {2015}}

@article{Mang:2017a,
	author = {Mang, A. and Ruthotto, L.},
	journal = {SIAM Journal on Scientific Computing},
	number = {5},
	pages = {B860--B885},
	title = {A {L}agrangian {G}auss--{N}ewton--{K}rylov solver for mass- and intensity-preserving diffeomorphic image registration},
	volume = {39},
	year = {2017}}

@article{Mang:2017b,
	author = {Mang, A. and Biros, G.},
	journal = {SIAM Journal on Scientific Computing},
	number = {6},
	pages = {B1064--B1101},
	title = {A semi-{L}agrangian two-level preconditioned {N}ewton--{K}rylov solver for constrained diffeomorphic image registration},
	volume = {39},
	year = {2017}}

@book{Nocedal:2006a,
	address = {New York, New York, US},
	author = {Nocedal, J. and Wright, S. J.},
	publisher = {Springer},
	title = {Numerical Optimization},
	year = {2006}}

@article{Osowski:2002a,
	author = {Osowski, S. and Nghia, D. D.},
	journal = {Pattern Recognition},
	number = {9},
	pages = {1949--1957},
	publisher = {Elsevier},
	title = {Fourier and wavelet descriptors for shape recognition using neural networks---a comparative study},
	volume = {35},
	year = {2002}}

@article{Wu:1993a,
	author = {Wu, W.-Y. and Wang, M.-J. J.},
	journal = {CVGIP: Graphical Models and Image Processing},
	number = {2},
	pages = {79--88},
	publisher = {Elsevier},
	title = {Detecting the dominant points by the curvature-based polygonal approximation},
	volume = {55},
	year = {1993}}

@article{Luciano:2019a,
	author = {Luciano, L. and Ben Hamza, A.},
	journal = {The Visual Computer},
	number = {6},
	pages = {1171--1180},
	publisher = {Springer},
	title = {Deep similarity network fusion for 3D shape classification},
	volume = {35},
	year = {2019}}

@article{Plotze:2005a,
	author = {Plotze, R. de O. and Falvo, M. and P{\'a}dua, J. G. and Bernacci, L. C. and Vieira, M. L. C. and Oliveira, G. C. X. and Bruno, O. M.},
	journal = {Canadian Journal of Botany},
	number = {3},
	pages = {287--301},
	publisher = {NRC Research Press Ottawa, Canada},
	title = {Leaf shape analysis using the multiscale {M}inkowski fractal dimension, a new morphometric method: {A} study with {P}assiflora ({P}assifloraceae)},
	volume = {83},
	year = {2005}}

@article{Torres:2004a,
	author = {Torres, R. S. and Falcao, A. X. and Costa, L. F.},
	journal = {Pattern Recognition},
	number = {6},
	pages = {1163--1174},
	publisher = {Elsevier},
	title = {A graph-based approach for multiscale shape analysis},
	volume = {37},
	year = {2004}}

@article{Junior:2018a,
	author = {Junior, J. J. and Backes, A. R. and Bruno, O. M.},
	journal = {Neurocomputing},
	pages = {201--209},
	publisher = {Elsevier},
	title = {Randomized neural network based descriptors for shape classification},
	volume = {312},
	year = {2018}}

@article{Huang:2021a,
	author = {Huang, H. and Amor, B. Ben and Lin, X. and Zhu, F. and Fang, Y.},
	journal = {arXiv preprint arXiv:2106.11911},
	title = {Residual Networks as Flows of Velocity Fields for Diffeomorphic Time Series Alignment},
	year = {2021}}

@inproceedings{Louis:2018a,
	author = {Louis, M. and Charlier, B. and Durrleman, S.},
	booktitle = {Proceedings of the IEEE Conference on Computer Vision and Pattern Recognition Workshops},
	pages = {332--340},
	title = {Geodesic discriminant analysis for manifold-valued data},
	year = {2018}}

@inproceedings{Mang:2016b,
	author = {Mang, A. and Gholami, A. and Biros, G.},
	booktitle = {Proc ACM/IEEE Conference on Supercomputing},
	pages = {842--853},
	title = {Distributed-memory large-deformation diffeomorphic {3D} image registration},
	year = {2016}}

@article{Gu:2018a,
	author = {Gu, J. and Wang, Z. and Kuen, J. and Ma, L. and Shahroudy, A. and Shuai, B. and Liu, T. and Wang, X. and Wang, G. and Cai, J. and others},
	journal = {Pattern Recognition},
	pages = {354--377},
	publisher = {Elsevier},
	title = {Recent advances in convolutional neural networks},
	volume = {77},
	year = {2018}}

@article{Yang:2017a,
	author = {X. Yang and R. Kwitt and M. Styner and M. Niethammer},
	journal = {NeuroImage},
	pages = {378-396},
	title = {Quicksilver: {F}ast predictive image registration---{A} deep learning approach},
	volume = {158},
	year = {2017}}

@incollection{Cutler:2012a,
	author = {Cutler, A. and Cutler, D. R. and Stevens, J. R.},
	booktitle = {Ensemble machine learning},
	pages = {157--175},
	publisher = {Springer},
	title = {Random forests},
	year = {2012}}

@book{Rosenblatt:1957a,
	author = {Rosenblatt, F.},
	publisher = {Cornell Aeronautical Laboratory},
	title = {The perceptron, a perceiving and recognizing automaton},
	year = {1957}}

@article{Cortes:1995a,
	author = {Cortes, C. and Vapnik, V.},
	journal = {Machine Learning},
	number = {3},
	pages = {273--297},
	publisher = {Springer},
	title = {Support-vector networks},
	volume = {20},
	year = {1995}}

@inproceedings{Ludke:2022a,
	author = {L{\"u}dke, D. and Amiranashvili, T. and Ambellan, F. and Ezhov, I. and Menze, B. H. and Zachow, S.},
	booktitle = {International Conference on Medical Image Computing and Computer-Assisted Intervention},
	organization = {Springer},
	pages = {453--463},
	title = {Landmark-Free Statistical Shape Modeling Via Neural Flow Deformations},
	year = {2022}}

@inproceedings{Davies:2002a,
	author = {Davies, R. H. and Twining, C. J. and Cootes, T. F. and Waterton, J. C. and Taylor, C. J.},
	booktitle = {European Conference on Computer Vision},
	organization = {Springer},
	pages = {3--20},
	title = {{3D} statistical shape models using direct optimisation of description length},
	year = {2002}}

@article{Davies:2002b,
	author = {Davies, R. H. and Twining, C. J. and Cootes, T. F. and Waterton, J. C. and Taylor, C. J.},
	journal = {IEEE Transactions on Medical Imaging},
	number = {5},
	pages = {525--537},
	publisher = {IEEE},
	title = {A minimum description length approach to statistical shape modeling},
	volume = {21},
	year = {2002}}

@incollection{Ambellan:2019a,
	author = {Ambellan, F. and Lamecker, H. and Tycowicz, C. von and Zachow, S.},
	booktitle = {Biomedical Visualisation},
	pages = {67--84},
	publisher = {Springer},
	title = {Statistical shape models: understanding and mastering variation in anatomy},
	year = {2019}}

@article{Heimann:2009a,
	author = {Heimann, T. and Meinzer, H.-P.},
	journal = {Medical Image Analysis},
	number = {4},
	pages = {543--563},
	publisher = {Elsevier},
	title = {Statistical shape models for 3D medical image segmentation: {A} review},
	volume = {13},
	year = {2009}}

@inproceedings{Mussabayeva:2018a,
	author = {Mussabayeva, A. and Kroshnin, A. and Kurmukov, A. and Denisova, Y. and Shen, L. and Cong, S. and Wang, L. and Gutman, B. A.},
	booktitle = {International Workshop on Shape in Medical Imaging},
	organization = {Springer},
	pages = {160--168},
	title = {Image registration and predictive modeling: {L}earning the metric on the space of diffeomorphisms},
	year = {2018}}

@article{Marslanda:2019a,
	author = {Marslanda, S. and Sommer, S.},
	journal = {Riemannian Geometric Statistics in Medical Image Analysis},
	pages = {135},
	publisher = {Academic Press},
	title = {Riemannian geometry on shapes and diffeomorphisms},
	year = {2019}}

@inproceedings{Lee:2020a,
	author = {Lee, B. C. and Tward, D. J. and Hu, Z. and Trouv{\'e}, A. and Miller, M. I.},
	booktitle = {Proceedings of the IEEE/CVF Conference on Computer Vision and Pattern Recognition Workshops},
	pages = {862--863},
	title = {Infinitesimal Drift Diffeomorphometry Models for Population Shape Analysis},
	year = {2020}}

@incollection{Guigui:2022a,
	author = {Guigui, N. and Pennec, X.},
	booktitle = {Handbook of Statistics},
	pages = {285--326},
	publisher = {Elsevier},
	title = {Parallel transport, a central tool in geometric statistics for computational anatomy: {A}pplication to cardiac motion modeling},
	volume = {46},
	year = {2022}}

@incollection{Bauer:2019a,
	author = {Bauer, M. and Charon, N. and Younes, L.},
	booktitle = {Handbook of Numerical Analysis},
	pages = {613--646},
	publisher = {Elsevier},
	title = {Metric registration of curves and surfaces using optimal control},
	volume = {20},
	year = {2019}}

@inproceedings{Bone:2018a,
	author = {B{\^o}ne, A. and Louis, M. and Martin, B. and Durrleman, S.},
	booktitle = {International Workshop on Shape in Medical Imaging},
	organization = {Springer},
	pages = {3--13},
	title = {Deformetrica 4: {A}n open-source software for statistical shape analysis},
	year = {2018}}

@incollection{Polzin:2020a,
	author = {Polzin, T. and Niethammer, M. and Vialard, F.-X. and Modersitzki, J.},
	booktitle = {Riemannian Geometric Statistics in Medical Image Analysis},
	pages = {479--532},
	publisher = {Elsevier},
	title = {A discretize--optimize approach for {LDDMM} registration},
	year = {2020}}

@article{Charon:2022a,
	author = {Charon, Nicolas and Younes, Laurent},
	journal = {arXiv preprint arXiv:2205.01237},
	title = {Shape spaces: {F}rom geometry to biological plausibility},
	year = {2022}}

@article{Krebs:2019a,
	author = {Krebs, J. and Delingette, H. and Mailh{\'e}, B. and Ayache, N. and Mansi, T.},
	journal = {IEEE Transactions on Medical Imaging},
	number = {9},
	pages = {2165--2176},
	publisher = {IEEE},
	title = {Learning a probabilistic model for diffeomorphic registration},
	volume = {38},
	year = {2019}}

@article{Bone:2020a,
	author = {B{\^o}ne, A. and Colliot, O. and Durrleman, S.},
	journal = {International Journal of Computer Vision},
	number = {12},
	pages = {2873--2896},
	publisher = {Springer},
	title = {Learning the spatiotemporal variability in longitudinal shape data sets},
	volume = {128},
	year = {2020}}

@inproceedings{Franccois:2021a,
	author = {Fran{\c{c}}ois, A. and Gori, P. and Glaun{\`e}s, J.},
	booktitle = {International Conference on Geometric Science of Information},
	organization = {Springer},
	pages = {781--788},
	title = {Metamorphic image registration using a semi-Lagrangian scheme},
	year = {2021}}

@inproceedings{Hsieh:2021a,
	author = {Hsieh, H.-W. and Charon, N.},
	booktitle = {International Conference on Information Processing in Medical Imaging},
	organization = {Springer},
	pages = {31--42},
	title = {Diffeomorphic registration with density changes for the analysis of imbalanced shapes},
	year = {2021}}

@article{Amor:2021a,
	author = {Amor, B. B. and Arguill{\`e}re, S. and Shao, L.},
	journal = {arXiv preprint arXiv:2102.07951},
	title = {{ResNet-LDDMM}: {A}dvancing the {LDDMM} framework using deep residual networks},
	year = {2021}}

@article{Bharati:2022a,
	author = {Bharati, S. and Mondal, M. and Podder, P. and Prasath, V. B.},
	journal = {arXiv preprint arXiv:2204.11341},
	title = {Deep Learning for Medical Image Registration: {A} Comprehensive Review},
	year = {2022}}

@inproceedings{Sun:2022a,
	author = {Sun, S. and Han, K. and Kong, D. and Tang, H. and Yan, X. and Xie, X.},
	booktitle = {Proceedings of the IEEE/CVF Conference on Computer Vision and Pattern Recognition},
	pages = {20845--20855},
	title = {Topology-Preserving Shape Reconstruction and Registration via Neural Diffeomorphic Flow},
	year = {2022}}

@article{Dupuis:1998a,
	author = {Dupuis, P. and Gernander, U. and Miller, M. I.},
	journal = {Quarterly of Applied Mathematics},
	number = {3},
	pages = {587--600},
	title = {Variational problems on flows of diffeomorphisms for image matching},
	volume = {56},
	year = {1998}}

@article{Hsieh:2022a,
	author = {Hsieh, D.-N. and Arguill{\`e}re, S. and Charon, N. and Younes, L.},
	journal = {SIAM Journal on Applied Dynamical Systems},
	number = {1},
	pages = {80--101},
	publisher = {SIAM},
	title = {Mechanistic Modeling of Longitudinal Shape Changes: equations of motion and inverse problems},
	volume = {21},
	year = {2022}}

@article{Hartman:2022a,
	author = {Hartman, E. and Sukurdeep, Y. and Klassen, E. and Charon, N. and Bauer, M.},
	journal = {arXiv preprint arXiv:2204.04238},
	title = {Elastic shape analysis of surfaces with second-order Sobolev metrics: a comprehensive numerical framework},
	year = {2022}}

@article{Bauer:2022a,
	author = {Bauer, M. and Charon, N. and Klassen, E. and Kurtek, S. and Needham, T. and Pierron, T.},
	journal = {arXiv preprint arXiv:2209.09862},
	title = {Elastic Metrics on Spaces of Euclidean Curves: Theory and Algorithms},
	year = {2022}}

@article{Miller:2002a,
	author = {M. I. Miller and A. Trouv\'e and L. Younes},
	journal = {Annual Review of Biomedical Engineering},
	number = {1},
	pages = {375--405},
	title = {On the metrics and {E}uler--{L}agrange equations of computational anatomy},
	volume = {4},
	year = {2002}}

@article{Younes:2009a,
	author = {Younes, L. and Arrate, F. and Miller, M.~I.},
	journal = {NeuroImage},
	pages = {S40--S50},
	title = {Evolutions equations in computational anatomy},
	volume = {45},
	year = {2009}}

@inproceedings{Shen:2005a,
	author = {Shen, D. G. and Sundar, H. and Xue, Z. and Fan, Y. and Litt, H.},
	booktitle = {Proc Medical Image Computing and Computer-Assisted Intervention},
	pages = {902--910},
	publisher = {Springer-Verlag Berlin},
	series = {Lecture Notes In Computer Science},
	title = {Consistent estimation of cardiac motions by {4D} image registration},
	volume = {3750},
	year = {2005}}

@article{Bai:2015a,
	author = {W. Bai and W. Shi and A. de Marvao and T. J. W. Dawes and D. P. O'Regan and S. A. Cook and D. Rueckert},
	journal = {Medical Image Analysis},
	number = {1},
	pages = {133--145},
	title = {A bi-ventricular cardiac atlas built from 1000+ high resolution {MR} images of healthy subjects and an alalysis of shape and motion},
	volume = {26},
	year = {2015}}

@inproceedings{Perperidis:2005a,
	author = {Perperidis, D. and Mohiaddin, R. and Rueckert, D.},
	booktitle = {Proc Medical Image Computing and Computer-Assisted Intervention},
	pages = {402--410},
	publisher = {Springer-Verlag Berlin},
	series = {Lecture Notes In Computer Science},
	title = {Construction of a {4D} statistical atlas of the cardiac anatomy and its use in classification},
	volume = {3750},
	year = {2005}}

@techreport{Breiman:1996a,
	author = {L. Breiman},
	institution = {Statistics Department, University of California Berkeley},
	title = {Out-of-bag estimation},
	year = {1996}}

@misc{rfpimp,
	author = {T. Parr and K. Turgutlu},
	title = {\texttt{rfpimp}},
	url = {https://pypi.org/project/rfpimp},
	year = {2022}}

@misc{rfpimp-git,
	author = {T. Parr and K. Turgutlu},
	title = {Feature Importances for \texttt{scikit}-learn machine learning models},
	url = {https://github.com/parrt/random-forest-importances},
	year = {2022}}

@article{Sotiras:2013a,
	author = {Sotiras, A. and Davatzikos, C. and Paragios, N.},
	journal = {Medical Imaging, IEEE Transactions on},
	number = {7},
	pages = {1153--1190},
	title = {Deformable medical image registration: {A} survey},
	volume = {32},
	year = {2013}}

@article{Himthani:2022a,
	author = {N. Himthani and M. Brunn and J. Y. Kim and M. Schulte and A. Mang and G. Biros},
	journal = {Journal of Imaging},
	number = {9},
	pages = {251},
	title = {{CLAIRE}: {P}arallelized diffeomorphic image registration for large-scale biomedical imaging applications},
	volume = {8},
	year = {2022}}

@article{Brunn:2021a,
	author = {M. Brunn and N. Himthani and G. Biros and M. Mehl and A. Mang},
	journal = {Journal of Parallel and Distributed Computing},
	pages = {149--162},
	title = {Fast {GPU} {3D} diffeomorphic image registration},
	volume = {149},
	year = {2021}}

@phdthesis{Freeman:2014a,
	author = {Freeman, J.},
	publisher = {University of Houston},
	title = {Combining diffeomorphic matching with image sequence intensity registration},
	year = {2014}}

@book{Jajoo:2011a,
	author = {Jajoo, A.},
	publisher = {University of Houston},
	title = {Diffeomorphic matching and dynamic deformable shapes},
	year = {2011}}

@article{Zekry:2012a,
	author = {Zekry, S. B. and Lawrie, G. and Little, S. and Zoghbi, W. and Freeman, J. and Jajoo, A. and Jain, S. and He, J. and Martynenko, A. and Azencott, R.},
	journal = {Cardiovascular Engineering and Technology},
	number = {4},
	pages = {402--412},
	title = {Comparative evaluation of mitral valve strain by deformation tracking in {3D}-echocardiography},
	volume = {3},
	year = {2012}}

@article{Zekry:2018a,
	author = {Zekry, S. B. and Freeman, J. and Jajoo, A. and He, J. and Little, S. H. and Lawrie, G. M. and Azencott, R. and Zoghbi, W. A.},
	journal = {JACC: Cardiovascular Imaging},
	number = {5},
	pages = {776--777},
	title = {Effect of mitral valve repair on mitral valve leaflets strain: {A} pilot study},
	volume = {11},
	year = {2018}}

@article{Zekry:2016a,
	author = {Zekry, S. B. and Freeman, J. and Jajoo, A. and He, J. and Little, S. H. and Lawrie, G. M. and Azencott, R. and Zoghbi, W. A.},
	journal = {Circulation: Cardiovascular Imaging},
	number = {1},
	pages = {e003254},
	title = {Patient-specific quantitation of mitral valve strain by computer analysis of three-dimensional echocardiography: a pilot study},
	volume = {9},
	year = {2016}}

@article{Parikh:2013a,
	author = {N. Parikh and S. Boyd},
	journal = {Foundations and Trends in Optimization},
	number = {3},
	pages = {123--231},
	title = {Proximal algorithms},
	volume = {1},
	year = {2013}}

@article{ODonoghue:2013a,
	author = {O'Donoghue, B. and Stathopoulos, G. and Boyd, S.},
	journal = {IEEE Transactions on Control Systems Technology},
	number = {6},
	pages = {2432--2442},
	publisher = {IEEE},
	title = {A splitting method for optimal control},
	volume = {21},
	year = {2013}}

@article{Boyd:2011a,
	author = {S. Boyd and N. Parikh and E. Chu and B. Peleato and J. Eckstein},
	journal = {Foundations and Trends in Machine Learning},
	number = {3},
	pages = {1--122},
	title = {Distributed optimization and statistical learning via the alternating direction method of multipliers},
	volume = {1},
	year = {2011}}

@article{Douglas:1956a,
	author = {Douglas, J. and Rachford, H. H.},
	journal = {Transactions of the American Mathematical Society},
	number = {2},
	pages = {421--439},
	title = {On the numerical solution of heat conduction problems in two and three space variables},
	volume = {82},
	year = {1956}}

@article{Gabay:1976a,
	author = {Gabay, D. and Mercier, B.},
	journal = {Computers \& Mathematics with Applications},
	number = {1},
	pages = {17--40},
	publisher = {Elsevier},
	title = {A dual algorithm for the solution of nonlinear variational problems via finite element approximation},
	volume = {2},
	year = {1976}}

@article{Glowinski:1975a,
	author = {Glowinski, R. and Marroco, A.},
	journal = {ESAIM: Mathematical Modelling and Numerical Analysis-Mod{\'e}lisation Math{\'e}matique et Analyse Num{\'e}rique},
	number = {R2},
	pages = {41--76},
	title = {Sur l'approximation, par {\'e}l{\'e}ments finis d'ordre un, et la r{\'e}solution, par p{\'e}nalisation-dualit{\'e} d'une classe de probl{\`e}mes de Dirichlet non lin{\'e}aires},
	volume = {9},
	year = {1975}}

@article{Grenander:1998a,
	author = {Grenander, U. and Miller, M. I.},
	journal = {Quarterly of Applied Mathematics},
	number = {4},
	pages = {617--694},
	title = {Computational anatomy: {A}n emerging discipline},
	volume = {56},
	year = {1998}}

@article{Miller:2004a,
	author = {Miller, M. I.},
	journal = {NeuroImage},
	number = {1},
	pages = {S19--S33},
	title = {Computational anatomy: {S}hape, growth and atrophy comparison via diffeomorphisms},
	volume = {23},
	year = {2004}}

@conference{Glaunes:2004a,
	author = {J. Glaunes and A. Trouv\'e and L. Younes},
	booktitle = {Proc IEEE Conference on Computer Vision and Pattern Recognition},
	pages = {712--718},
	title = {Diffeomorphic matching of distributions: {A} new approach for unlabelled point-sets and sub-manifolds matching},
	volume = {2},
	year = {2004}}

@article{Beg:2005a,
	author = {M. F. Beg and M. I. Miller and A. Trouv\'e and L. Younes},
	journal = {International Journal of Computer Vision},
	number = {2},
	pages = {139--157},
	title = {Computing large deformation metric mappings via geodesic flows of diffeomorphisms},
	volume = {61},
	year = {2005}}

@article{Gorce:1996a,
	author = {Gorce, J.-M. and Friboulet, D. and Magnin, I. E.},
	journal = {Medical Image Analysis},
	number = {3},
	pages = {245--261},
	title = {Estimation of three-dimensional cardiac velocity fields: {A}ssessment of a differential method and application to three-dimensional {CT} data},
	volume = {1},
	year = {1996}}

@article{Delingette:2012a,
	author = {Delingette, H. and Billet, F. and Wong, K. C. L. and Sermesant, M. and Rhode, K. and Ginks, M. and Rinaldi, C. and Razavi, R. and Ayache, N.},
	journal = {IEEE Transactions on Biomedical Engineering},
	number = {1},
	pages = {20--24},
	title = {Personalization of cardiac motion and contractility from images using variational data assimilation},
	volume = {59},
	year = {2012}}

@article{Miller:2001a,
	author = {Miller, M. I. and Younes, L.},
	journal = {International Journal of Computer Vision},
	number = {1/2},
	pages = {61--81},
	title = {Group actions, homeomorphism, and matching: {A} general framework},
	volume = {41},
	year = {2001}}

@article{Chawla:2002a,
	author = {Chawla, N. V. and Bowyer, K. W. and Hall, L. O. and Kegelmeyer, P.},
	journal = {Journal of artificial intelligence research},
	pages = {321--357},
	title = {SMOTE: synthetic minority over-sampling technique},
	volume = {16},
	year = {2002}}

@article{Bonan:1990a,
	author = {Bonan, S. S. and Clark, D. S.},
	journal = {Journal of Approximation Theory},
	number = {2},
	pages = {210--224},
	title = {Estimates of the Hermite and the Freud polynomials},
	volume = {63},
	year = {1990}}

@article{Azencott:2010a,
	author = {Azencott, R. and Glowinski, R. and He, J. and Jajoo, A. and Li, Y. and Martynenko, A. and Hoppe, R. H. W. and Benzekry, S. and Little, S. H.},
	journal = {Computational Methods in Applied Mathematics},
	number = {3},
	pages = {235--274},
	title = {Diffeomorphic matching and dynamic deformable surfaces in 3D medical imaging},
	volume = {10},
	year = {2010}}

@article{Vadakkumpadan:2012a,
	author = {Vadakkumpadan, F. and Arevalo, H. and Ceritoglu, C. and Miller, M. and Trayanova, N.},
	journal = {IEEE Transactions on Medical Imaging},
	number = {5},
	pages = {1051--1060},
	title = {Image-based estimation of ventricular fiber orientations for personalized modeling of cardiac electrophysiology},
	volume = {31},
	year = {2012}}

@inproceedings{Lombaert:2011a,
	author = {Lombaert, H. and Peyrat, J.-M. and Croisille, P. and Rapacchi, S. and Fanton, L. and Clarysse, P. and Delingette, H. and Ayache, N.},
	booktitle = {International Conference on Functional Imaging and Modeling of the Heart},
	pages = {171--179},
	title = {Statistical analysis of the human cardiac fiber architecture from {DT-MRI}},
	year = {2011}}

@inproceedings{Sundar:2009a,
	author = {Sundar, H. and Davatzikos, C. and Biros, G.},
	booktitle = {Proc Medical Image Computing and Computer-Assisted Intervention},
	pages = {257--265},
	title = {Biomechanically constrained {4D} estimation of mycardial motion},
	volume = {LNCS 5762},
	year = {2009}}

@article{Mansi:2011a,
	author = {Mansi, T. and Pennec, X. and Sermesant, M. and Delingette, H. and Ayache, N.},
	journal = {International Journal of Computer Vision},
	number = {1},
	pages = {92--111},
	title = {{iLogDemons}: {A} demons-based registration algorithm for tracking incompressible elastic biological tissues},
	volume = {92},
	year = {2011}}

@article{Bistoquet:2008a,
	author = {Bistoquet, A. and Oshinski, J. and Skrinjar, O.},
	journal = {Medical Image Analysis},
	number = {1},
	pages = {69--85},
	title = {Myocardial deformation recovery from cine {MRI} using a nearly incompressible biventricular model},
	volume = {12},
	year = {2008}}

@article{Mang:2019a,
	author = {Mang, A. and Gholami, A. and Davatzikos, C. and Biros, G.},
	journal = {SIAM Journal on Scientific Computing},
	number = {5},
	pages = {C548--C584},
	title = {{CLAIRE}: {A} distributed-memory solver for constrained large deformation diffeomorphic image registration},
	volume = {41},
	year = {2019}}

@inproceedings{Brunn:2020a,
	author = {Brunn, M. and Himthani, N. and Biros, G. and Mehl, M. and Mang, A.},
	booktitle = {SC20: International Conference for High Performance Computing, Networking, Storage and Analysis},
	organization = {IEEE},
	pages = {1--17},
	title = {Multi-node multi-{GPU} diffeomorphic image registration for large-scale imaging problems},
	year = {2020}}

@book{Modersitzki:2009a,
	address = {Philadelphia, Pennsylvania, US},
	author = {Modersitzki, J.},
	publisher = {SIAM},
	title = {{FAIR}: Flexible algorithms for image registration},
	year = {2009}}

@book{Modersitzki:2004a,
	address = {New York},
	author = {Modersitzki, J.},
	publisher = {Oxford University Press},
	title = {Numerical methods for image registration},
	year = {2004}}

@article{Glaunes:2008a,
	author = {J. Glaun\`es and A. Qiu and M. I. Miller and L. Younes},
	journal = {International Journal of Computer Vision},
	number = {3},
	pages = {317--336},
	title = {Large deformation diffeomorphic metric curve mapping},
	volume = {80},
	year = {2008}}

@book{Younes:2019a,
	author = {L. Younes},
	edition = {2},
	publisher = {Springer Verlag Berlin Heidelberg},
	title = {Shapes and diffeomorphisms},
	volume = {171},
	year = {2019}}

@article{Trouve:1998a,
	author = {A. Trouv\'e},
	journal = {International Journal of Computer Vision},
	number = {3},
	pages = {213--221},
	title = {Diffeomorphism groups and pattern matching in image analysis},
	volume = {28},
	year = {1998}}

@article{Zhang:2021a,
	author = {Zhang, P. and Mang, A. and He, J. and Azencott, R. and El-Tallawi, K. C. and Zoghbi, W. A.},
	journal = {Journal of Optimization Theory and Applications},
	number = {1},
	pages = {143--168},
	title = {Diffeomorphic Shape Matching by Operator Splitting in {3D} Cardiology Imaging},
	volume = {188},
	year = {2021}}

@article{ElTallawi:2021a,
	author = {El-Tallawi, K. C. and Zhang, P. and Azencott, R. and He, J. and Herrera, E. L. and Xu, J. and Chamsi-Pasha, M. and Jacob, J. and Lawrie, G. M. and Zoghbi, W. A.},
	journal = {Cardiovascular Imaging},
	number = {6},
	pages = {1099--1109},
	title = {Valve Strain Quantitation in Normal Mitral Valves and Mitral Prolapse With Variable Degrees of Regurgitation},
	volume = {14},
	year = {2021}}

@article{ElTallawi:2021b,
	author = {El-Tallawi, K. C. and Zhang, P. and Azencott, R. and He, J. and Xu, J. and Herrera, E. L. and Jacob, J. and Chamsi-Pasha, M. and Lawrie, G. M. and Zoghbi, W. A.},
	journal = {Cardiovascular Imaging},
	number = {4},
	pages = {782--793},
	title = {Mitral valve remodeling and strain in secondary mitral regurgitation: {C}omparison with primary regurgitation and normal valves},
	volume = {14},
	year = {2021}}

@article{ElTallawi:2019a,
	author = {El-Tallawi, K. C. and Zhang, P. and Azencott, R. and He, J. and Herrera, E. and Chamsi-Pasha, M. and Jacob, J. and Lawrie, G. M. and Zoghbi, W.},
	journal = {Journal of the American College of Cardiology},
	number = {9S1},
	pages = {1953--1953},
	title = {Quantitation of mitral valve strain in normals and in patients with mitral valve prolapse},
	volume = {73},
	year = {2019}}

@inproceedings{Ho:1995a,
	author = {Ho, T. K.},
	booktitle = {Proceedings of 3rd International Conference on Document Analysis and Recognition},
	pages = {278--282},
	title = {Random decision forests},
	volume = {1},
	year = {1995}}

@article{Breiman:2001a,
	author = {Breiman, L.},
	journal = {Machine learning},
	number = {1},
	pages = {5--32},
	title = {Random forests},
	volume = {45},
	year = {2001}}

@article{fasshauer2011positive,
	author = {Fasshauer, Gregory E},
	journal = {Dolomites Research Notes on Approximation},
	pages = {21--63},
	title = {Positive definite kernels: past, present and future},
	volume = {4},
	year = {2011}}

@article{krasikov2004new,
	author = {Krasikov, Ilia},
	journal = {arXiv preprint math/0401310},
	title = {New bounds on the Hermite polynomials},
	year = {2004}}

@misc{stein1999interpolation,
	author = {Stein, ML},
	publisher = {Springer New York},
	title = {Interpolation of spatial data. Springer series in statistics},
	year = {1999}}

\end{document}